\newif\ifdraft \drafttrue
\documentclass[11pt]{article}

\usepackage{parskip}
\usepackage[ruled]{algorithm2e}

\usepackage{fullpage}
\usepackage{amsmath}
\usepackage{natbib}
\usepackage{amsfonts}
\usepackage{amsthm}

\usepackage{mathtools}
\usepackage{color}
\usepackage{xcolor}
\usepackage{multicol}
\usepackage{bbm}
\usepackage{hyperref}
\usepackage{mleftright}
\usepackage{authblk}
\usepackage{esvect}

\newcommand{\xhdr}[1]{\vspace{2mm} \noindent{\bf #1}}

 \usepackage{color-edits}
\addauthor{sw}{blue}
\addauthor{yb}{orange}

\newcommand\cF{\mathcal{F}}
\newcommand\cX{\mathcal{X}}
\newcommand\cY{\mathcal{Y}}
\newcommand\cD{\mathcal{D}}
\newcommand{\E}{\mathop{\mathbb{E}}}
\newcommand{\cR}{\mathbb{R}}
\newcommand{\Lagr}{\mathcal{L}}
\DeclareMathOperator*{\argmin}{\mathrm{argmin}}

\newcommand{\ind}{\mathbbm{1}}
\newcommand{\pair}{\rho}
\newcommand{\nullPair}{null}
\usepackage{thmtools,thm-restate}
\newcommand{\R}{\cR}

\newtheorem{theorem}{Theorem}[section]
\newtheorem{observation}[theorem]{Observation}

\newtheorem{corollary}[theorem]{Corollary}
\newtheorem{remark}[theorem]{Remark}
\newtheorem{definition}[theorem]{Definition} 

\newtheorem{lemma}[theorem]{Lemma} 

\newcommand{\bx}{x}
\newcommand{\by}{y}

\newcommand{\Auditor}{\mathcal{J}}
\newcommand{\Algorithm}{\mathcal{A}}

\newcommand{\metric}{d}
\newcommand{\fviolation}{v}
\newcommand{\Regret}{\mathbf{Regret}}

\newcommand{\round}{t}
\newcommand{\totalRound}{T}
\newcommand{\subRound}{\tau}
\newcommand{\totalSubRound}{k}

\newcommand{\interpolation}{q}
\newcommand{\cH}{\mathcal{H}}

\newcommand{\norm}[1]{\left\lVert#1\right\rVert}

\newcommand{\misclassLoss}{\text{Err}}
\newcommand{\fairLoss}{\text{Unfair}}
\DeclarePairedDelimiter\floor{\lfloor}{\rfloor}

\newcommand{\cZ}{\mathcal{Z}}
\newcommand{\onlineLoss}{L}
\newcommand{\acc}{\text{acc}}
\newcommand{\fair}{\text{fair}}
\newcommand{\seed}{\text{seed}}
\newcommand{\cP}{\mathcal{P}}

\newcommand{\fairbatch}{\textsc{fair-batch}}
\newcommand{\batch}{\textsc{batch}}
\newcommand{\policyContext}{\xi}
\newcommand{\policyContextClass}{\Xi}
\newcommand{\policy}{\psi}
\newcommand{\policyClass}{\Psi}
\newcommand{\ftpl}{\textsc{context-ftpl}}
\newcommand{\arbitraryInstance}{v}

\usepackage{times}

 \title{Metric-Free Individual Fairness in Online Learning\footnote{Previous versions of this paper included an error in one of the proofs, which was fixed using an additional assumption in the most recent version. In this version, we correct the error without the need for any additional assumptions. However, we achieve slightly slower rates than before. We were also made aware of the connection to the online optimization with long-term constraints literature, and the new related work section discusses some similarities and differences.}}
 \author{Yahav Bechavod\thanks{ The Hebrew University. Email: \texttt{yahav.bechavod@cs.huji.ac.il}.} \qquad Christopher Jung\thanks{University of Pennsylvania. Email: \texttt{chrjung@seas.upenn.edu}.} \qquad Zhiwei Steven Wu\thanks{ Carnegie Mellon University. Email: \texttt{zstevenwu@cmu.edu}. }}

\begin{document}

\maketitle

\begin{abstract}
We study an online learning problem subject to the constraint of individual fairness, which requires that similar individuals are treated similarly.  Unlike prior work on individual fairness, we do not assume the similarity measure among individuals is known, nor do we assume that such measure takes a certain parametric form.  Instead, we leverage the existence of an \emph{auditor} who detects fairness violations without enunciating the quantitative measure. In each round, the auditor examines the learner's decisions and attempts to identify a pair of individuals that are treated unfairly by the learner. We provide a general reduction framework that reduces online classification in our model to standard online classification, which allows us to leverage existing online learning algorithms to achieve sub-linear regret and number of fairness violations. Surprisingly, in the stochastic setting where the data are drawn independently from a distribution, we are also able to establish PAC-style fairness and accuracy generalization guarantees~(\cite{YonaR18}), despite only having access to a very restricted form of fairness feedback. Our fairness generalization bound qualitatively matches the uniform convergence bound of~\cite{YonaR18}, while also providing a meaningful accuracy generalization guarantee. Our results resolve an open question by~\cite{GillenJKR18} by showing that online learning under an unknown individual fairness constraint is possible even without assuming a strong parametric form of the underlying similarity measure.
\end{abstract}

\thispagestyle{empty} \setcounter{page}{0}
\clearpage

\tableofcontents
\clearpage

\section{Introduction}\label{sec:introduction}
As machine learning increasingly permeates many critical aspects of society, including education, healthcare, criminal justice, and lending, there is by now a vast literature that studies how to make machine learning algorithms fair (see, e.g., \cite{ChouR}; \cite{podesta2014big}; \cite{CorbettDaviesGo18}). Most of the work in this literature tackles the problem by taking the \emph{statistical group fairness} approach that first fixes a small collection of high-level groups defined by protected attributes (e.g., race or gender) and then asks for approximate parity of some statistic of the predictor, such as positive classification rate or false positive rate, across these groups (see, e.g., \cite{hardt16, chouldechova, kleinberg2017inherent, AgarwalBD0W18}). While notions of group fairness are easy to operationalize, this approach is aggregate in nature and does not provide fairness guarantees for finer subgroups or individuals \citep{individualfairness, Hebert-JohnsonK18, KearnsNRW18}.

In contrast, the \emph{individual fairness} approach aims to address this limitation by asking for explicit fairness criteria at an individual level. In particular, the compelling notion of individual fairness proposed in the seminal work of \cite{individualfairness} requires that similar people are treated similarly. The original formulation of individual fairness assumes that the algorithm designer has access to a task-specific fairness metric that captures how similar two individuals are in the context of the specific classification task at hand. In practice, however, such a fairness metric is rarely specified, and the lack of  metrics has been a major obstacle for the wide adoption of individual fairness. 
There has been recent work on learning the fairness metric based on different forms of human feedback. For example, \cite{Ilvento} provides an algorithm for learning the  metric by presenting human arbiters with queries concerning the distance between individuals, and \cite{GillenJKR18} provide an online learning algorithm that can eventually learn a Mahalanobis metric based on identified fairness violations. While these results are encouraging, they are still bound by several limitations. In particular, it might be difficult for humans to enunciate a precise quantitative similarity measure between individuals. Moreover, their similarity measure across individuals may not be consistent with any metric (e.g., it may not satisfy the triangle inequality) and is unlikely to be given by a simple parametric form (e.g., Mahalanobis metric function).

To tackle these issues, this paper studies \emph{metric-free} online learning algorithms for individual fairness that rely on a weaker form of interactive human feedback and minimal assumptions on the similarity measure across individuals. Similar to the prior work of \cite{GillenJKR18}, we do not assume a pre-specified metric, but instead assume access to an \emph{auditor}, who observes the learner's decisions over a group of individuals that show up in each round and attempts to identify a fairness violation---a pair of individuals in the group that should have been treated more similarly by the learner. Since the auditor only needs to identify such unfairly treated pairs, there is no need for them to enunciate a quantitative measure -- to specify the distance between the identified pairs. Moreover, we do not impose any parametric assumption on the underlying similarity measure, nor do we assume that it is actually a metric since we do not require that similarity measure to satisfy the triangle inequality.
Under this model, we provide a general reduction framework that can take any online classification algorithm (without fairness constraint) as a black-box and obtain a learning algorithm that can simultaneously minimize cumulative classification loss and the number of fairness violations. Our results in particular remove many strong assumptions in \cite{GillenJKR18}, including their parametric assumptions on linear rewards and Mahalanobis distances, and answer several questions left open in their work. We now provide an overview of the results.

\subsection{Overview of Model and Results}
We study an online classification problem: over rounds $\round = 1, \ldots ,\totalRound$, a learner observes a small set of $\totalSubRound$ individuals\footnote{For ease of presentation, we present our results for the case where $k=O(1)$, which is relatively standard in the relevant online learning literature.} with their feature vectors  $(x_\subRound^\round )_{\subRound=1}^\totalSubRound$ in space $\cX$. The learner tries to predict the label $y^\round_\totalSubRound \in \{0, 1\}$ of each individual with a ``soft'' predictor $\pi^\round$ that predicts $\pi^\round(x_\subRound^\round)\in [0,1]$ on each $x_\subRound^\round$ and incurs classification loss $\vert\pi^\round(x_\subRound^\round) - y_\subRound^\round\vert$. Then an auditor investigates whether the learner has violated the individual fairness constraint on any pair of individuals within this round, that is, if there exists
 $({\subRound_1}, {\subRound_2}) \in [\totalSubRound]^2$ such that $\vert\pi^\round(x_{\subRound_1}^\round) - \pi^\round(x_{\subRound_2}^\round)| > \metric(x_{\subRound_1}^\round, x_{\subRound_2}^t) + \alpha$, where $\metric\colon \cX \times \cX \rightarrow \mathbb{R}_{+}$ is an  unknown distance function and $\alpha$ denotes the auditor's tolerance. If such a violation has occurred on any number of pairs, the auditor identifies one such pair and incurs a fairness loss of 1; otherwise, the fairness loss is 0. The learner then updates the predictive policy based on the observed labels and the received fairness feedback. Under this model, our results include:
\begin{itemize}
    \item \textbf{A General Reduction From Metric-Free Online Classification to Standard Online Classification}
    
    Our reduction-based algorithm can take any no-regret online (batch) classification learner as a black-box and achieve sub-linear cumulative fairness loss and sub-linear regret on mis-classification loss compared to the most accurate policy that is fair on every round. In particular, our framework can leverage the generic exponential weights method \citep{FreundS97,cesa,arora2012multiplicative} and also oracle-efficient methods, including variants of Follow-the-Perturbed-Leader (FTPL) (e.g., \cite{SyrgkanisKS16, ftpl2, AgarwalGH19}), that further reduces online learning to standard supervised learning or optimization problems. We instantiate our framework using two online learning algorithms (exponential weights and $\ftpl$), both of which obtain a $O(\sqrt{T})$ regret rate.

    \item \textbf{Fairness and Accuracy Generalization Guarantee of the Learned Policy}    
    While our algorithmic results hold under adversarial arrivals of the individuals, in the stochastic arrivals setting we show that the uniform average policy over time is probably approximate correct and fair (PACF)~\citep{YonaR18}---that is, the policy is approximately fair on almost all random pairs drawn from the distribution and nearly matches the accuracy guarantee of the best fair policy. In particular, we show that the average policy $\pi^{avg}$ with high probability satisfies 
    \[\Pr_{x,x'}[|\pi^{avg}(x) - \pi^{avg}(x') | > \alpha + T^{-c}] \leq O(T^{-c})\]
    for some constant $c>0$. In other words, we achieve non-trivial PACF uniform convergence sample complexity as \cite{YonaR18}, albeit a slightly worse rate.\footnote{\cite{YonaR18} show that if a policy $\pi$ is $\alpha$-fair on all pairs in a i.i.d. dataset of size $m$, then $\pi$ satisfies $\Pr_{x,x'}[|\pi(x) - \pi(x')| > \alpha + \epsilon]\leq \epsilon$, as long as $m\geq \tilde\Omega(1/\epsilon^4)$. } However, we establish our generalization guarantee through fundamentally different techniques. While their work assumes a fully specified metric and i.i.d. data, the learner in our setting can only access the similarity measure through an auditor's limited fairness violations feedback.
    The main challenge we need to overcome is that the fairness feedback is inherently adaptive---that is, the auditor only provides feedback for the sequence of deployed policies, which are updated adaptively over rounds. In comparison, a fully known metric allows the learner to evaluate the fairness guarantee of all policies simultaneously. Therefore, we cannot rely on their uniform convergence result to bound the fairness generalization error, but instead we leverage a probabilistic argument that relates the learner's regret to the distributional fairness guarantee.
\end{itemize}

\subsection{Related Work}

\xhdr{Solving open problems in \cite{GillenJKR18}.} The most related work to ours is \cite{GillenJKR18}, which studies the linear contextual bandit problem subject to individual fairness constraints with respect to an unknown Mahalanobis metric. Similar to our work, they also assume an auditor who can identify fairness violations in each round and provide an online learning algorithm with sublinear regret and a bounded number of fairness violations. Our results resolve two main questions left open in their work. First, we assume a weaker auditor who is only required to identify a single fairness violation (as opposed to all of the fairness violations in their setting). Second, we remove the strong parametric assumption regarding the Mahalanobis metric and instead work with a broad class of similarity functions that need not take a metric form.

Starting with \cite{JosephKMR16}, there is a different line of work that studies online learning for individual fairness, but subject to a different notion called meritocratic fairness \citep{JabbariJKMR17,JosephKMNR18, KannanKMPRVW17}. These results present algorithms that are ``fair'' within each round but again rely on strong realizability assumptions--their fairness guarantee depends on the assumption that the outcome variable of each individual is given by a linear function. \cite{GuptaK19} also studies online learning subject to individual fairness but with a known metric. They formulate a one-sided fairness constraint across time, called fairness in hindsight, and provide an algorithm with regret $O(T^{M/(M+1)})$ for some distribution-dependent constant $M$. 

Our work is related to several others that aim to enforce individual fairness without a known metric. \cite{Ilvento} studies the problem of metric learning by asking human arbiters distance queries. Unlike \cite{Ilvento}, our algorithm does not explicitly learn the underlying similarity measure and does not require asking auditors numeric queries. The PAC-style fairness generalization bound in our work falls under the framework of \emph{probably approximately metric-fairness} due to \cite{YonaR18}. However, their work assumes a pre-specified fairness metric and i.i.d. data from a distribution, while we establish our generalization through a sequence of adaptive fairness violations feedback over time. \cite{KimRR18} study a group-fairness relaxation of individual fairness, which requires that similar subpopulations are treated similarly. They do not assume a pre-specified metric for their offline learning problem, but they do assume a metric oracle that returns numeric distance values on random pairs of individuals.  \cite{subjectiveFairness} study an offline learning problem with subjective individual fairness, in which the algorithm tries to elicit subjective fairness feedback from human judges by asking them questions of the form ``should this pair of individuals be treated similarly or not?" Their fairness generalization takes a different form, which involves taking averages over both the individuals and human judges. We aim to provide a fairness generalization guarantee that holds for almost all individuals from the population.

Our technique of combining the loss and the constraint violation into a Lagrangian loss is very related to the technique used in the online convex optimization with long term constraints \cite{jenatton2016adaptive, 0002N20, YuNW17, CaoL19, MahdaviJY12}. Similarly as in our setting, they are interested in choosing some point $x^t$ in each round and incur some loss $f^t(x^t)$ where $f^t$ is chosen by the adversary. Given some collection of constraints $\{g_i\}_{i=1}^m$, the goal is to simultaneously ensure the total violation of the constraints $\sum_{t=1}^t \sum_{i=1}^m g_i(x^t)$ is sublinear and to achieve sublinear regret against any fixed point $x^*$ in hindsight that satisfies all the constraints --- $\sum_{i=1}^m g_i(x^*) \le 0$. However, in most of these settings, the constraints are initially known to the learner, and the algorithm requires a projection into the feasible region. This is different than our setting, where we have no knowledge of what the space of fair policies looks like, as the auditor cannot enunciate what the fairness metric is. The only exception among the works cited above is the work of \citet{CaoL19} and \citet{MahdaviJY12}. They consider a bandit feedback setting where the constraints are not known initially and only the max violation of the constraints with respect to the chosen point is revealed in each round --- i.e. $\max_{i \in [m]} g_i(x^t)$. Our setting is still different in that we do not receive the amount of violation in each round but only the point's feasibility (i.e. $\ind(\max_{i \in [m]} g_i(x^t) \le 0)$) and one of the violated constraints chosen arbitrarily. Furthermore, in their setting, the the point chosen in each round $x^t$ is a $d$-dimesnional vector (i.e. $x^t \in \cR^d$), whereas we imagine the policy chosen in each round comes from the simplex of some hypothesis class $\Delta \cF$, which is often much bigger than $\cR^d$. In that sense, our use of the Follow-The-Perturbed-Leader approach to make the overall algorithm oracle-efficient is novel.

Our techniques in establishing generalization bounds by leveraging the regret guarantees of online algorithms are closely related to classical online to batch conversion schemes (most notably \cite{HelmboldWarmuth, CesaBianchi04}). However, since our fairness violation loss function is threshold-based, and hence not linear with respect to the base classifiers over which the deployed policy is defined, we cannot apply standard analysis regarding the fairness error of the average policy (\cite{CesaBianchi04}), since, in the worst case, this error rate could be equal to the sum of the error rates of all deployed policies during the run of the online algorithm. Another classical alternative is to randomly select one of the policies in the run of the algorithm (\cite{HelmboldWarmuth}). However, in order to establish a high probability bound, this results in a rate incorporating a linear dependence on the inverse of the failure probability, which we prefer to avoid. In this regard, our generalization technique is novel in its ability to incorporate knowledge about the specifics of the fairness loss in relating the overall number of fairness violations encountered during the run of the algorithm to the probability of encountering slightly larger violations when deploying the average policy over its run.

\section{Model and Preliminaries}
We define the instance space to be $\cX$ and its label space to be $\cY$. Throughout this paper, we will restrict our attention to binary labels, that is $\cY=\{0,1\}$. Sometimes, we assume the
existence of an underlying (but unknown) distribution $\cD$ over $\cX\times\cY$. We write $\cF: \cX \to \cY$ to denote the hypothesis class and assume that $\cF$ contains a constant hypothesis --- i.e. there exists $f$ such that $f(x)=0$ (and/or 1) for all $x \in \cX$. 
Also, we allow a convex combination of hypotheses for the purpose of randomizing the prediction and denote the simplex of hypotheses by $\Delta \cF$. For consistency, we say $f \in \cF$, which maps to $\cY = \{0,1\}$, is a \emph{hypothesis} and $\pi \in \Delta \cF$, which is a mixture of some hypotheses and hence maps to $[0,1]$, a \emph{policy}. For each prediction $\hat{y} \in \cY$ and its true label $y \in \cY$, there is an associated misclassification loss, $\ell(\hat{y}, y) = \ind[\hat{y} \neq y]$. For simplicity, we overload the notation and write 
\[
\ell(\pi(x), y)  = (1-\pi(x)) \cdot y  + \pi(x) \cdot (1-y)= \E_{f \sim \pi}[\ell(f(x), y)].
\]
Note that the loss is linear in $\pi$. Given $\pi_1$ and $\pi_2$ and some mixing probability $p$, define $\pi_3(x) = p \pi_1(x) + (1-p)(\pi_2(x))$. Then, it is immediate that 
\begin{align*}
    \ell(\pi_3(x), y) = p \ell(\pi_1(x), y) + (1-p) \ell(\pi_2(x), y).
\end{align*}

\subsection{Online Batch Classification}
Here, we describe the vanilla online batch classification setting, as we will try to reduce the problem we are interested in to this setting. In each round $\round=1, \dots, \totalRound$, a learner deploys a policy $\pi^\round \in \Delta \cF$. Upon seeing the deployed policy $\pi^\round$, the environment chooses a batch of $\totalSubRound$ individuals, $ (x^\round_\subRound,y^\round_\subRound)_{\subRound=1}^\totalSubRound$. As the environment can adversarially and adaptively choose this batch of individuals, we can think of this as the environment strategically choosing $z_{\batch}^\round = (x^t, y^t)$ where
we write $\bx^\round = (x^\round_\subRound)_{\subRound=1}^\totalSubRound$ and $\by^\round = (y^\round_\subRound)_{\subRound=1}^\totalSubRound$ for simplicity. Note that the strategy space for the environment is
\[
\cZ_\batch^\round = \cX^k \times \cY^k.
\]
The history in each round then will consist of everything observed by the learner up to, but not including, round $t$:
\[
    h^t = ((\pi^1, z^1), \dots, (\pi^{t-1}, z^{t-1})).
\]
The space of such histories is denoted by $\cH_\batch^t = (\Delta \cF \times \cZ_{\batch})^{t-1}$. A learner $\Algorithm: \cF^*_{\batch} \to \Delta \cF$ is then defined to be a mapping from history to a policy:
\[
    \pi^t = \Algorithm(h^{t-1}).
\]

Given some loss that is calculated in each round $t$, the learner cannot hope to upper-bound the loss by itself over all rounds $t=1, \dots, T$, due to the adversarial nature of the environment. Therefore, the learner can only hope to minimize its regret with respect to some fixed policy it could have used.
\begin{definition}
Fix the adversary's strategy space $\cZ$. With respect to some baseline $Q \subseteq \Delta \cF$ and some loss $\onlineLoss: \Delta \cF \times \cZ \to \R$, we say learner $\Algorithm$'s regret with respect to adaptively and adversarially chosen sequence of $(z^t)_{t=1}^T$ is 
\[
    \sum_{\round=1}^\totalRound \onlineLoss\left(\pi^\round, z^\round\right) - \min_{\pi^* \in Q} \sum_{\round=1}^\totalRound \onlineLoss\left(\pi^*, z^\round\right).
\]
\end{definition}

In this vanilla online batch classification setting, performance is measured solely using the misclassification loss:
\begin{definition}[Misclassification Loss]
The (batch) misclassification loss $\misclassLoss$ is 
\[
\misclassLoss(\pi, z^t) =  \sum_{\subRound=1}^{\totalSubRound} \ell(\pi(x^\round_{\subRound}), y^\round_{\subRound}).
\]
\end{definition}
In other words, the goal in this setting is to come up with an learner $\Algorithm$ such that against any adaptively and adversarially chosen $(z^t)_{t=1}^T$, we can achieve 
\[
\sum_{\round=1}^\totalRound \misclassLoss\left(\pi^\round, z^\round\right) - \min_{\pi^* \in Q} \sum_{\round=1}^\totalRound \misclassLoss\left(\pi^*, z^\round\right) = o(T).
\]
Often, when learner $\Algorithm$ can guarantee that the regret is sublienar as above, it is said to be a no-regret learner. Examples of such no-regret learners include exponential weights \citep{FreundS97, cesa, arora2012multiplicative} and Follow-the-Perturbed-Leader strategies \cite{SyrgkanisKS16, ftpl2}. 

\subsection{Online Fair Batch Classification}
As opposed to only attempting to minimize its misclassification regret, we also want to make sure that the deployed policies $(\pi^t)_{t=1}^T$ satisfy individual fairness constraints (i.e. each policy $\pi^t$ treats similar individuals similarly according to some fairness metric $d: \cX \times \cX \to \R$). 

\begin{definition}[$\alpha,\beta$-fairness]
Assume $\alpha,\beta > 0$. A policy $\pi \in \Delta \cF$ is said to be $\alpha$-fair on pair $(x,x')$, if \[
|\pi(x) - \pi(x')| \le d(x,x') + \alpha.
\]
We say policy $\pi$'s $\alpha$-fairness violation on pair $(x,x')$ is \[\fviolation_\alpha(\pi, (x,x'))=\max(0, |\pi(x) - \pi(x')| - d(x,x') - \alpha).\]
A policy $\pi$ is said to be $(\alpha, \beta)$-fair on distribution $\cD$ over $\cX\times\cY$, if
\[
\Pr_{(x,x')\sim\cD\vert_{\cX}\times\cD\vert_{\cX}}[\vert\pi(x)-\pi(x')\vert > d(x,x')+\alpha] \leq \beta.
\]
\end{definition} 

In order to find such fair policies, we rely on an auditor who can give feedback by pointing out when a pair of two similar individuals are not treated similarly according to metric $d$. However, unlike \cite{GillenJKR18}, we make no parametric assumption on the metric $d$ nor do we require it to be a proper metric (i.e. it doesn't need to satisfy the triangle inequality). The only requirement regarding $\metric$ in our work is that it is non-negative and symmetric\footnote{From a technical standpoint, there is no necessity in our framework for $\metric$ to be symmetric. However, we incorporate this assumption as we interpret $\metric$ as a distance function.}. In addition to removing the parametric assumption on the metric $d$, we further require the auditor to output only one arbitrary pair where a fairness violation has occurred as opposed to reporting all violations.
\begin{definition}[Auditor]
An auditor $\Auditor_{d,\alpha}$, who can have its own internal state, takes in a \emph{reference set} $S \subseteq \cX$ and a policy $\pi$. Then, it outputs $\pair$ which is either $\nullPair$ or a pair of indices from the provided reference set to denote that there is an $\alpha$-fairness violation on that pair.
For some $x=(x_1, \dots x_k)$,
\[
    \Auditor_{d, \alpha}(x, \pi) = 
        \begin{cases}
            \pair=(\pair_1, \pair_2) & \text{if} \quad \exists \pair_1,\pair_2 \in [k]. \pi(x_{\pair_1}) - \pi(x_{\pair_2}) > d(x_{\pair_1}, x_{\pair_2}) + \alpha\\
            \nullPair & \text{otherwise}\\
        \end{cases}
\]
If there exist multiple pairs on which there is an $\alpha$-violation, the auditor can choose one arbitrarily. We will elide $d$ and write $\Auditor_{\alpha}$, as we will only focus on the case where $d$ is fixed. 
\end{definition}
\begin{remark}
We emphasize that the assumptions on the auditor here are greatly relaxed in comparison to the assumptions made in \cite{GillenJKR18}, which require that the auditor outputs whether the policy is $0$-fair (i.e. with no slack) on all pairs exactly. Furthermore, the auditor from \cite{GillenJKR18} can only handle Mahalanobis distances. In our setting, because of the internal state of the auditor, the auditor does not have to be a fixed function but rather can be adaptively changing in each round. Our argument actually never relies on the fact the distance function $\metric$ stays the same throughout rounds, meaning all our results extend to the case where the distance function governing the fairness constraints is changing every round. For simplicity, we focus on the case where $d$ is fixed.
\end{remark}

The order in which the learner, the environment, and the auditor interact is as follows. In each round $t=1, \dots, T$, a learner deploys a policy $\pi^t \in \Delta \cF$. Then, a batch of $k$ individuals $(x^t,y^t) = ((x^t_i)_{i \in [k]}, (y^t_i)_{i \in [k]})$ arrives. The auditor, upon inspecting $(\pi^t, x^t)$, provides a fairness feedback $\rho^t \in [\totalSubRound]^2 \cup \{\nullPair\}$ which may be a pair of indices for which the deployed policy $\pi^t$ is treating unfairly or null if there does not exist any such pair. As the auditor can essentially choose its reported pair arbitrarily among all pairs on which a violation exists, we can actually fold the auditor into the environment. That is -- the environment choose $z_{\fairbatch}^t = (x^t, y^t, \rho^t)$ simultaneously, meaning the strategy space of the environment here is \[
\cZ_\fairbatch=(\cX^k \times \cY^k \times ([\totalSubRound]^2 \cup \{\nullPair\})).
\]
Similarly as in the vanilla online batch classification setting, the history of the interaction is now described as 
\[
    h^t = ((\pi^1, z^1), \dots, (\pi^{t-1}, z^{t-1})).
\]
where $z^t \in \cZ_{\fairbatch}$ and the space of such histories is then $\cH_\fairbatch^t = (\Delta \cF \times \cZ_{\fairbatch} )^{t-1}$. A learner $\Algorithm$ is still defined to be a mapping from a history to a policy:
\[
    \pi^t = \Algorithm(h^{t-1}).
\]

In addition to the misclassification loss, the learner is also has penalized for unfairness, using the unfairness loss.
\begin{definition}[Unfairness Loss]
The $\alpha$-fairness loss $\fairLoss_\alpha$ is
\[
\fairLoss_\alpha(\pi, z^\round) = \begin{cases} 
\ind\left[\fviolation_\alpha(\pi, (x^t_{\rho^t_1}, x^t_{\rho^t_2})) > 0\right]
& \text{if $\pair^\round = (\pair^\round_1, \pair^\round_2)$}\\
0 & \text{otherwise}
\end{cases}
\]
\end{definition}
In other words, the learner incurs misclassification loss $\misclassLoss(\pi^t, z^t)$\footnote{We overload the notation for $\misclassLoss$ and write $\misclassLoss(\pi, ((x, y), \rho)) = \misclassLoss(\pi, (x,y))$.} and unfairness loss $\fairLoss(\pi^t, z^t)$ in each round $t$. Note that unlike the fairness loss defined in \cite{GillenJKR18} which counts the total number of pairs with fairness violations, the fairness loss here is either 0 or 1 depending on the existence of such pair in this setting. We compare this online batch classification setting with fairness constraints to the vanilla online batch classification setting in Figure \ref{fig:comparison-fair-batch-and-batch}.

\begin{figure}
\begin{minipage}[t]{.53\textwidth}
\begin{algorithm}[H]
\caption{Online Fair Batch Classification $\fairbatch$}
\SetAlgoLined
 \For{$\round=1,\dots,\totalRound$}{
    Learner deploys $\pi^\round$\\
    Environment chooses $(\bx^\round, \by^\round)$\\
    Environment chooses the pair $\pair^\round$\\
    $z^\round = ((\bx^\round, \by^\round), \pair^\round)$ \\
    Learner incurs misclassfication loss $\misclassLoss(\pi^\round, z^\round)$\\
    Learner incurs fairness loss $\fairLoss(\pi^\round, z^\round)$\\
    Learner receives $(x^t, y^t)$ and $\rho^t$
  }
\end{algorithm}
\end{minipage}
\hfill
\begin{minipage}[t]{.43\textwidth}
\begin{algorithm}[H]
\caption{Online Batch Classification $\batch$}
\SetAlgoLined
 \For{$\round=1,\dots,\totalRound$}{
    Learner deploys $\pi^\round$\\
    Environment chooses $z^\round = (\bx^\round, \by^\round)$\\
    Learner incurs misclassification loss $\misclassLoss(\pi^t, z^\round)$\\
    Learner receives $(x^t, y^t)$.
  }
\end{algorithm}
\end{minipage}
\caption{Comparison between Online Fair Batch Classification and Online Batch Classification: each is summarized by the interaction between the learner and the environment: $(\Delta \cF \times \cZ_{\fairbatch})^\totalRound$ and  $(\Delta \cF \times \cZ_{\batch})^\totalRound$ where $\cZ_{\fairbatch} = \cX^\totalSubRound \times \cY^\totalSubRound \times ([k]^2 \cup \{null\})$ and $\cZ_{\batch} = \cX^\totalSubRound \times \cY^\totalSubRound$.}\label{fig:comparison-fair-batch-and-batch}
\end{figure}

Finally, the baseline $Q_\alpha$ that we compete against will be all policies that are $\alpha$-fair on $\bx^\round$ for all $\round \in [\totalRound]$:
\[
Q_\alpha = \{\pi \in \Delta \cF: \text{$\pi$ is $\alpha$-fair on $\bx^\round$ for all $\round \in [\totalRound]$}\}.
\]
Given some fixed trade-off slack $\epsilon \in (0, \alpha)$ and an auditor $\Auditor_{\alpha}$, our goal is to provide a learner $\Algorithm$ such that for any adversarially and adaptively chosen $((x^t, y^t))_{t=1}^T$, regret with respect to each of misclassification and unfairness loss is sublinear against $Q_{\alpha-\epsilon}$:
\begin{enumerate}
    \item Misclassification regret:
    \[\sum_{\round=1}^\totalRound \misclassLoss\left(\pi^\round, z^\round\right) - \min_{\pi^* \in Q_{\alpha-\epsilon}} \sum_{\round=1}^\totalRound \misclassLoss\left(\pi^*, z^\round\right) = o(T).\]
    \item Unfairness regret: \[\sum_{\round=1}^\totalRound  \fairLoss_{\alpha}(\pi^\round, z^\round) - \min_{\pi^* \in Q_{\alpha-\epsilon}}\sum_{\round=1}^\totalRound  \fairLoss_{\alpha}(\pi^*, z^\round) = \sum_{\round=1}^\totalRound  \fairLoss_{\alpha}(\pi^\round, z^\round) = o(T).\]
\end{enumerate}
where $z^t = (x^t, y^t, \Auditor_\alpha(x^t, \pi^t))$. Because $\pi^* \in Q_{\alpha-\epsilon}$ must be $\alpha$-fair, guaranteeing sublinear fairness regret is equivalent to guaranteeing the overall fairness loss is sublinear.

\section{Lagrangian Regret}
As we wish to achieve no-regret with respect to both the misclassification and fairness loss, it is natural to consider a hybrid loss that combines them together. In fact, we define a round-based Lagrangian loss and show that regret with respect to our Lagrangian loss also serves as an upperbound for the misclassification and the unfairness regret multipled by some parameter that balances the misclassification and unfairness losses in the Lagrangian loss.  

Then, we show how to achieve no regret with respect to the Lagrangian loss by reducing the problem to an online batch classification where there is no fairness constraint. For concreteness, we show how to leverage exponential weights in order to achieve sublinear misclassification regret and fairness loss. 

\subsection{Lagrangian Formulation}\label{subsec:lagr-form}
Here we present a hybrid loss that we call \emph{Lagrangian loss} that combines the misclassification loss and the fairness violation in round $t$.

\begin{definition}[Lagrangian Loss]
\label{def:Lag-Loss}
We say that the $C$-Lagrangian loss of $\pi$ is \[\Lagr_{C}\left(\pi, \left((\bx^\round,\by^\round), \pair^\round\right) \right) = \sum_{\subRound=1}^{\totalSubRound} \ell\left(\pi\left(x^\round_{\subRound}\right), y^\round_{\subRound}\right) + \begin{cases} 
          C \left(\pi(x^\round_{\pair_1}) - \pi(x^\round_{\pair_2}) \right) &  \pair^\round = (\pair_1, \pair_2) \\
          0 & \pair^\round = \nullPair
       \end{cases}
\]
\end{definition}

Given an auditor $\Auditor_\alpha$ that is tasked with detecting any $\alpha$-fairness violation, we can simulate the online fair batch classification setting with an auditor $\Auditor_{\alpha}$ by setting the pair $\pair_{\Auditor_{\alpha}}^\round = \Auditor_{\alpha}(\bx^\round, \pi^t)$.

Now, we show that the Lagrangian regret serves as an upper bound for the sum of the $\alpha$-fairness loss (with some multiplicative factor that depends on $C$ and $\epsilon$) and the misclassification loss regret with respect to $Q_{\alpha-\epsilon}$.
\begin{theorem}
\label{thm:fair-misclass-lagr}
Given an auditor $\Auditor_\alpha$, fix any sequence $((x^t, y^t))_{t=1}^T$, $(\pi^t)_{t=1}^T$, and $\rho^t_{\Auditor_\alpha} = \Auditor_{\alpha}(x^t, \pi^t)$ for each $t \in [T]$. Then, the following holds for any $\epsilon \in [0, \alpha]$ and $\pi^* \in Q_{\alpha-\epsilon}$ simultaneously:
\begin{align*}
    & C\epsilon \sum_{\round=1}^\totalRound  \fairLoss_{\alpha}(\pi^\round, z^t) +  \left(\sum_{\round=1}^\totalRound \misclassLoss(\pi^t, z^t) - \sum_{\round=1}^\totalRound \misclassLoss(\pi^*, z^t)\right) \\
    &\le \sum_{\round=1}^\totalRound \Lagr_{C}(\pi^\round, z^t) -  \sum_{\round=1}^\totalRound \Lagr_{C}(\pi^*, z^t)
\end{align*}
where $z^t = ((x^t, y^t), \rho_{\Auditor_\alpha}^t)$ for each $t \in [T]$.
\end{theorem}
\begin{proof}
Fix $\epsilon \in [0,\alpha]$. Fix any $(\alpha-\epsilon)$-fair policy $\pi^* \in Q_{\alpha-\epsilon}$.
Note that for any round $t$ where $\rho^t_{\Auditor_\alpha} \neq \nullPair$, we have
\begin{align*}
    \pi^*(x^\round_{\pair^\round_1}) - \pi^*(x^\round_{\pair^\round_2}) \le \metric(x^\round_{\pair^\round_1}, x^\round_{\pair^\round_2}) + \alpha- \epsilon \quad&\Rightarrow\quad-\left(\pi^*(x^\round_{\pair^\round_1}) - \pi^*(x^\round_{\pair^\round_2}) - \metric(x^\round_{\pair^\round_1}, x^\round_{\pair^\round_2}) -\alpha \right)\ge \epsilon\\
    \pi^t(x^\round_{\pair^\round_1}) - \pi^t(x^\round_{\pair^\round_2}) > \metric(x^\round_{\pair^\round_1}, x^\round_{\pair^\round_2}) + \alpha \quad&\Rightarrow\quad
    \pi^t(x^\round_{\pair^\round_1}) - \pi^t(x^\round_{\pair^\round_2}) - \metric(x^\round_{\pair^\round_1}, x^\round_{\pair^\round_2}) -\alpha > 0
\end{align*}
because $\pi^*$ is $(\alpha-\epsilon)$-fair on this pair and $\pi^t$ is not $\alpha$-fair on $(x^t_{\rho^t_1}, x^t_{\rho^t_2})$.

Then we can show
\begin{align*}
    &\sum_{\round=1}^{\totalRound} \Lagr_{C}(\pi^\round, z^t) - \Lagr_{C}(\pi^*, z^t) \\
    &= \sum_{\round=1}^{\totalRound} \sum_{\subRound=1}^{\totalSubRound} \ell\left(\pi^\round \left(x^\round_{\subRound}\right), y^\round_{\subRound}\right) - \ell\left(\pi^* \left(x^\round_{\subRound}\right), y^\round_{\subRound}\right) + \sum_{\round \in [\totalRound]: \pair_{\Auditor_\alpha}^\round \neq \nullPair} C\left(\pi^t(x^\round_{\pair^\round_1}) - \pi^t(x^\round_{\pair^\round_2})  \right)\\
    &-C\left(\pi^*(x^\round_{\pair^\round_1}) - \pi^*(x^\round_{\pair^\round_2}) \right)\\
    &= \sum_{\round=1}^{\totalRound} \sum_{\subRound=1}^{\totalSubRound} \ell\left(\pi^\round \left(x^\round_{\subRound}\right), y^\round_{\subRound}\right) - \ell\left(\pi^* \left(x^\round_{\subRound}\right), y^\round_{\subRound}\right) + \sum_{\round \in [\totalRound]: \pair_{\Auditor_\alpha}^\round \neq \nullPair} C\left(\pi^t(x^\round_{\pair^\round_1}) - \pi^t(x^\round_{\pair^\round_2}) - \metric(x^\round_{\pair^\round_1}, x^\round_{\pair^\round_2}) -\alpha \right)\\
    &-C\left(\pi^*(x^\round_{\pair^\round_1}) - \pi^*(x^\round_{\pair^\round_2}) - \metric(x^\round_{\pair^\round_1}, x^\round_{\pair^\round_2}) -\alpha \right)\\
    &\ge \sum_{\round=1}^{\totalRound}\sum_{\subRound=1}^{\totalSubRound} \ell\left(\pi^\round \left(x^\round_{\subRound}\right), y^\round_{\subRound}\right) - \ell\left(\pi^* \left(x^\round_{\subRound}\right), y^\round_{\subRound}\right) + \sum_{\round \in [\totalRound]: \pair_{\Auditor_\alpha}^\round \neq \nullPair} C \epsilon\\
    &= \sum_{\round=1}^{\totalRound} \sum_{\subRound=1}^{\totalSubRound}\ell\left(\pi^\round \left(x^\round_{\subRound}\right), y^\round_{\subRound}\right) - \ell\left(\pi^* \left(x^\round_{\subRound}\right), y^\round_{\subRound}\right) + C\epsilon \sum_{\round=1}^\totalRound  \fairLoss_{\alpha}(\pi^\round, ((x^t, y^t), \rho_{\Auditor_\alpha}^t))
\end{align*}
\end{proof}

By considering the above theorem statement with $\epsilon=0$, we can always bound the misclassification regret with the Lagrangian regret:
\begin{align*}
\max_{\pi^* \in Q_{\alpha}}\left(\sum_{\round=1}^\totalRound \misclassLoss(\pi^t, z^t) - \sum_{\round=1}^\totalRound \misclassLoss(\pi^*, z^t)\right) &\le \max_{\pi^* \in Q_{\alpha}} \sum_{\round=1}^\totalRound \Lagr_{C}(\pi^\round, z^t) -  \sum_{\round=1}^\totalRound \Lagr_{C}(\pi^*, z^t) \\
&\le \max_{f^* \in \cF} \sum_{\round=1}^\totalRound \Lagr_{C}(\pi^\round, z^t) -  \sum_{\round=1}^\totalRound \Lagr_{C}(\pi^*, z^t) 
\end{align*}
 
However, note that we cannot always hope to bound the fairness loss with the Lagrangian regret because the misclassification regret may be negative. It alludes to the fact that it is not sufficient to simply come up with a way to bound the Lagrangian regret to be sublinear in $T$. In fact, we have to tune
$C$ according to the exact Lagrangian regret guarantee we wish to get: we want to set $C$ to be large enough so as to give more emphasis to the fairness loss. At the same time, we cannot set it too high, as $C$ also controls the range of the Lagrangian loss and the regret guarantees usually has a linear dependence on the range of the loss. In the next section, we will show exactly how to set $C$ so that we can achieve no regret with respect to each of the unfairness and misclassification losses simultaneously.

\section{Achieving No Regret Simultaneously}
\label{sec:noregret}
In Section \ref{subsec:online-batch-reduction}, we present an efficient reduction from the setting of online batch classification \emph{with} fairness constraints to that of online batch classification \emph{without} constraints. This reduction to the online batch classification without the fairness constraints is not necessary, as the Lagrangian loss is already linear in the first argument, $\pi$. However, we go through this reduction, as it lets us apply off-the-shelf no-regret algorithms for the online learning setting without fairness constraints.

Then, combining our reduction to the online batch classification without fairness constraints, exponential weights approach, and $C$ that is appropriately set with respect to the regret guarantee of exponential weights, we show how to bound each of the misclassification regret and unfairness regret with $O(T^\frac{3}{4})$ in Section \ref{subsec:exp-weights}.

Finally, in Section \ref{subsec:ftpl}, we show how we can use the Follow-The-Perturbed-Leader approach, specifically that of \citet{SyrgkanisKS16}, can be utilized to construct an algorithm which operates in an oracle-efficient manner.

\subsection{Reduction to Online Batch Classification without Fairness Constraints}
\label{subsec:online-batch-reduction}
Here we show how to reduce the online batch fair classification problem to the online batch classification problem in an efficient manner. Once this reduction is complete, we can leverage any online batch algorithm $\Algorithm_{\batch}((\pi^\tau, (\bx'^\tau, \by'^\tau))_{\tau=1}^{\round})$ as a black box in order to achieve sublinear Lagrangian regret. At a high level, our reduction involves just carefully transforming our online fair batch classification history up to $\round$, $(\pi^\tau, ((\bx^\tau, \by^\tau), \pair_{\Auditor}^\tau))_{\tau=1}^{\round} \in (\Delta  \cF \times \cZ_{\fairbatch})^\round$ into some fake online batch classification history $\left(\pi^\tau, (\bx'^\tau, \by'^\tau)\right)_{\tau=1}^{\round} \in (\Delta  \cF \times \cZ_{\batch})^\round$ and then feeding this artificially created history to a no-regret learner $\Algorithm_\batch$ for the online batch classification setting. 

Without loss of generality, we assume that $C$ is an integer; if it's not, then take the ceiling. Now, we describe how the transformation of the history works. For each round $\round$, whenever $\pair^\round = (\pair^\round_1, \pair^\round_2)$, we add $C$ copies of each of $(x^\round_{\pair^\round_1},0)$ and $(x^\round_{\pair^\round_2},1)$ to the original pairs to form $\bx'^\round$ and $\by'^\round$. In order to keep the batch size the same across each round, even if $\pair^\round = \nullPair$, we add $C$ copies of each of $(\arbitraryInstance, 0)$ and $(\arbitraryInstance, 1)$ where $\arbitraryInstance$ is some arbitrary instance in $\cX$. We describe this process in more detail in Algorithm \ref{alg:reduction}. 

\begin{algorithm}[t]
\caption{Reduction from Online Fair Batch Classification to Online Batch Classification, $R_{C}((x^t, y^t), \rho^t)$}
\label{alg:reduction}
\SetAlgoLined
 \SetKwInput{KwInput}{Parameters}
 \KwInput{$C$}
 \SetKwInput{KwInput}{Input}
 \KwInput{$(x^t, y^t), \rho^t$}
    \If{$\pair^\round = (\pair_1^\round, \pair_2^\round)$}{
        \For{$i = 1, \dots, C$}{
            $x^\round_{\totalSubRound + i} = x^\round_{\pair_1^\round} \quad \text{and}\quad y^\round_{\totalSubRound + i} = 0$;\\
            $x^\round_{\totalSubRound + C + i} = x^\round_{\pair_2^\round} \quad \text{and}\quad y^\round_{\totalSubRound + C + i} = 1$;\\
            }
    }
    \Else{
        \For{$i = 1, \dots, C$}{
            $x^\round_{\totalSubRound + i} = \arbitraryInstance \quad \text{and}\quad y^\round_{\totalSubRound + i} = 0$;\\
            $x^\round_{\totalSubRound + C + i} = \arbitraryInstance \quad \text{and}\quad y^\round_{\totalSubRound + C + i} = 1$;\\
        }
    }
    $\bx'^\round = (x^\round_{\subRound})_{\subRound=1}^{k+2C} \quad \text{and} \quad \by'^\round = (y^\round_{\subRound})_{\subRound=1}^{k+2C};$
    
    \SetKwInput{KwOutput}{Output}
    \KwOutput{$({x'}^t, {y'}^t)$}
\end{algorithm}

This reduction essentially preserves the regret, which we formally state in Lemma \ref{lem:reduction_regret}.
\begin{lemma}
\label{lem:reduction_regret}
For any sequence of $(\pi^\round)_{\round=1}^\totalRound$,  $((x^t, y^t))_{\round=1}^\totalRound$, $(\rho^t)_{t=1}^T$, and $\pi^* \in \Delta  \cF$, we have
 \[\sum_{\round=1}^\totalRound \Lagr_{C}(\pi^t, z^\round) - \sum_{\round=1}^\totalRound \Lagr_{C}(\pi^*, z^\round) = \sum_{\round=1}^{\totalRound} \misclassLoss(\pi^t, ({x'}^t, {y'}^t)) - \sum_{\round=1}^{\totalRound} \misclassLoss(\pi^*, ({x'}^t, {y'}^t))
 \]
 where $z^t = ((x^t, y^t), \rho^t)$ and $({x'}^t, {y'}^t) = R_C((x^t, y^t), \rho^t)$.
\end{lemma}
\begin{proof}
It is sufficient to show that in each round $\round \in [T]$, \[
\Lagr_{C,\alpha}(\pi^t, z^\round) - \Lagr_{C,\alpha}(\pi^*, z^\round) =  \sum_{\subRound=1}^{\totalSubRound+2C} \ell(\pi^t(x^t_{\subRound}), y^\round_{\subRound}) - \sum_{\subRound=1}^{\totalSubRound+2C} \ell(\pi^*(x^t_{\subRound}), y^\round_{\subRound})
\]
First, assume $\pair^\round = (\pair^\round_1, \pair^\round_2)$.
\begin{align*}
    &\Lagr_{C}(\pi^t, z^\round) - \Lagr_{C}(\pi^*, z^\round) \\
    &= \left(\sum_{\subRound=1}^{\totalSubRound} \ell(\pi^\round(x^\round_{\subRound}), y^\round_{\subRound}) + C(\pi^\round(x^\round_{\pair^\round_1}) - \pi^\round(x^\round_{\pair^\round_2}))\right) - \left(\sum_{\subRound=1}^{\totalSubRound} \ell(\pi^*(x^\round_{\subRound}), y^\round_{\subRound}) + C(\pi^*(x^\round_{\pair^\round_1}) - \pi^*(x^\round_{\pair^\round_2}))\right)\\
    &= \left(\sum_{\subRound=1}^{\totalSubRound} \ell(\pi^\round(x^\round_{\subRound}), y^\round_{\subRound}) + \left(\sum_{\subRound=1}^{C} \ell(\pi^\round(x^\round_{\pair^\round_1}), 0)
    +\sum_{\subRound=1}^{C} \ell(\pi^\round(x^\round_{\pair^\round_2}), 1) - C \right) \right)\\
    &- \left(\sum_{\subRound=1}^{\totalSubRound} \ell(\pi^*(x^\round_{\subRound}), y^\round_{\subRound}) + \left(\sum_{\subRound=1}^{C} \ell(\pi^*(x^\round_{\pair^\round_1}), 0)
    +\sum_{\subRound=1}^{C} \ell(\pi^*(x^\round_{\pair^\round_2}), 1) - C\right)  \right)\\
    &= \sum_{\subRound=1}^{\totalSubRound+2C} \ell(\pi^\round(x'^\round_{\subRound}), y'^\round_{\subRound}) - \sum_{\subRound=1}^{\totalSubRound+2C} \ell(\pi^*({x'}^t_{\subRound}), y'^\round_{\subRound}),
\end{align*}

The second equality follows from the fact that for any $\pi$ and $x$,
\[\ell(\pi(x), 0) = \pi(x)\quad\text{and}\quad \ell(\pi(x), 1) = 1-\pi(x).\]

If $\pair^\round = \nullPair$, the same argument as above applies; the only difference is that all the $\pi^t(\arbitraryInstance)$ cancel with each other because the number of copies with label 0 is exactly the same as that of label 1.
\end{proof}

\subsection{Exponential Weights}
\label{subsec:exp-weights}
It is well known that for linear loss functions, exponential weights with appropriately tuned learning rate $\gamma$ can achieve no regret \citep{FreundS97,cesa,arora2012multiplicative}. We will first describe the setting of the exponential weights. We will then show how the setting of online batch classification we are interested in can be cast to this setting.

For each round $t=1, \dots, T$:
\begin{enumerate}
    \item The learner chooses a distribution $p^t=(p^t_1, \dots, p^t_N)$ over $N$ experts.
    \item The adversary, \emph{with full knowledge of $p^t$}, chooses loss vector $m^t = (m^t_1, \dots, m^t_N)$ where $m^t_i \in [-B, B]$ for each $i \in [N]$.
    \item The learner suffers $p^t \cdot m^t$.
\end{enumerate}

We emphasize that since the adversary chooses $m^t$ with full knowledge of $p^t$, $m^t$ can be chosen as a function of $p^t$.

Exponential weights is defined as $p^t = \frac{1}{N}$ for $t=1$ and for any $t\ge 2$ 
\[
    \hat{p}^{t+1}_i = (1-\gamma m^t_i) p^t_i
\]
and \[
    p^{t+1} = \frac{\hat{p}^{t+1}_i}{\sum_{i \in [N]}\hat{p}^{t+1}_i}.
\]

Then, we have the following guarantee on the regret of exponential weights:
\begin{theorem}[\citet{arora2012multiplicative}]
\label{thm:exp-weights}
Exponential weights with learning rate $\gamma$ has the following guarantee: for any sequence of $(m^t)_{t=1}^T$ and any $i \in [N]$,
\[
 \sum_{t=1}^T p^t \cdot m^t \le  \sum_{t=1} m^t_i + B\left(\gamma \totalRound  + \frac{\ln(|\cF|)}{\gamma}\right).
\]
In other words, with learning rate $\gamma= \sqrt{\frac{\ln( N)}{\totalRound}}$, the regret is $2B\sqrt{\ln(N)T}$.
\end{theorem}

We can easily reduce the online classification setting to that of exponential weights method. Each $\pi^t$ that we deploy can be represented as a probability distribution $p^t=(p^t_f)_{f \in \cF}$ over each $f \in \cF$: for any $x \in \cX$,
\begin{align*}
    \pi^t(x) = \sum_{f \in \cF} p^t_f f(x).
\end{align*}

If we use $m^t_f = \misclassLoss(f, ({x'}^t, {y'}^t))$ for every $f \in \cF$. Then, we have for any $({x'}^t, {y'}^t) \in (\cX^{k+2C} \times \cY^{k+2C})$,
\begin{align*}
    \misclassLoss(\pi^t(x), {x'}^t, {y'}^t) = \sum_{f \in \cF} p^t_f \misclassLoss(f, ({x'}^t, {y'}^t))= p^t \cdot m^t.
\end{align*}
Putting everything together, we have \begin{enumerate}
    \item The learner deploys $\pi^t$ where the associated probability distribution over $\cF$ is $p^t = (p^t_f)_{f \in \cF}$.
    \item The adversary, with the knowledge of $p^t$, selects $({x'}^t, {y'}^t)$ which determines $m^t = (m^t_f)_{f\in \cF}$ where
    $m^t_f = \misclassLoss(f, ({x'^t},{y'}^t))$. Remember that $({x'}^t, {y'}^t) = R_C(x^t, y^t, \Auditor_\alpha(x^t, \pi^t))$ is a function of $\pi^t$.
    \item Learner suffers 
    \begin{align*}
        p^t \cdot m^t = \misclassLoss(\pi^t(x), {x'}^t, {y'}^t).
    \end{align*}
\end{enumerate}

Note that $m^t_f \in [-(k+2C), k+2C]$ for any $f \in \cF$. Although $({x'}^t, {y'}^t)$ is formed as a function of $\pi^t$, the exponential weights approach still allows the adversary to form $m^t$ as a function of $p^t$ or equivalently $\pi^t$. Therefore, we can use the regret guarantee of the exponential weights and appropriately set $C$ to achieve sublinear fairness loss and misclassification regret.

\begin{theorem}
\label{thm:exp-guarantee}
If $C = \max(T^{\frac{1}{4}},k)$, exponential weights guarantees the following: for any adaptively and adversarially chosen $((x^t, y^t))_{t=1}^T$, we have,
\begin{align*}
\sum_{\round=1}^\totalRound  \fairLoss_{\alpha}(\pi^\round, z^\round) &\le \frac{6}{\alpha}\sqrt{\ln(|\cF|)T} + \frac{k}{\alpha}T^{\frac{3}{4}}\\
\sum_{t=1}^T \misclassLoss(\pi^t, z^t) - \min_{\pi^* \in Q_{\alpha}}\sum_{t=1}^T \misclassLoss(\pi^*, z^t) &\le 6\sqrt{\ln(|\cF|)}T^{\frac{3}{4}}
\end{align*}
where $z^t = ((x^t, y^t), \rho^t_{\Auditor_\alpha})$ and $\rho^t_{\Auditor_\alpha} = \Auditor(x^t, \pi^t)$. In other words, misclassification regret and unfairness regret are both bounded by $O(T^{\frac{3}{4}})$.
\end{theorem}
\begin{proof}
First, we apply the regret guarantee of exponential weights. Theorem \ref{thm:exp-weights} gives us that 
\begin{align*}
    \sum_{t=1}^T \misclassLoss(\pi^t, ({x'}^t, {y'}^t)) - \min_{f^* \in \cF} \sum_{t=1}^T \misclassLoss(f^*, ({x'}^t, {y'}^t)) \le 2(k+2C)\sqrt{\ln(|\cF|)T}.
\end{align*}
because $\misclassLoss(\pi^t, ({x'}^t,{y'}^t)) \in [-(k+2C), k+2C]$. Note that $({x'}^t, {y'}^t) = R_C(x^t, y^t, \Auditor_\alpha(x^t, \pi^t))$ is a function of $\pi^t$, but as we emphasized before, the regret guarantee still holds.

Combining our previous lemmas and theorems, we have for any $((x^t, y^t))_{t=1}^T$ and hence $z^t = (x^t, y^t, \Auditor_{\alpha}(x^t, \pi^t))$ for $t\in [T]$,
\begin{align*}
    &C\epsilon \sum_{\round=1}^\totalRound  \fairLoss_{\alpha}(\pi^\round, z^t) + \left(\sum_{\round=1}^\totalRound \misclassLoss(\pi^t, z^t) - \min_{\pi^* \in Q_{\alpha-\epsilon}} \sum_{\round=1}^\totalRound \misclassLoss(\pi^*, z^t)\right) \\
    &\le \sum_{\round=1}^\totalRound \Lagr_{C, \alpha}(\pi^\round, z^t) - \min_{\pi^* \in Q_{\alpha-\epsilon}} \sum_{\round=1}^\totalRound \Lagr_{C, \alpha}(\pi^*, z^t) &&\text{Theorem \ref{thm:fair-misclass-lagr}} \\
    &\le \sum_{\round=1}^\totalRound \Lagr_{C, \alpha}(\pi^\round, z^t) - \min_{\pi^* \in \Delta \cF} \sum_{\round=1}^\totalRound \Lagr_{C, \alpha}(\pi^*, z^t) &&Q_{\alpha-\epsilon} \subseteq \Delta \cF\\
    &= \sum_{\round=1}^\totalRound \misclassLoss(\pi^t, ({x'}^t, {y'}^t)) - \min_{\pi^* \in \Delta \cF} \sum_{t=1}^T\misclassLoss(\pi^*, ({x'}^t, {y'}^t)) &&\text{Lemma \ref{lem:reduction_regret}}\\
    &= \sum_{\round=1}^\totalRound \misclassLoss(\pi^t, ({x'}^t, {y'}^t)) - \min_{f^* \in \cF} \sum_{t=1}^T\misclassLoss(f^*, ({x'}^t, {y'}^t)) &&\text{Linearity of $\sum_{t=1}^T\misclassLoss(\cdot, ({x'}^t, {y'}^t))$}\\
    &\le 2(k+2C)\sqrt{\ln(|\cF|)T}.
\end{align*}
simultaneously for all $\epsilon \in [0,\alpha]$.

For the fairness loss, consider $\epsilon = \alpha$, and fix any $\pi^* \in Q_{\alpha-\epsilon}$. We then have
    \begin{align*}
        &\sum_{t=1}^T \misclassLoss(\pi^t, (x^t, y^t)) - \sum_{\round=1}^\totalRound \misclassLoss(\pi^*, (x^t, y^t)) + C\alpha \sum_{\round=1}^\totalRound  \fairLoss_{\alpha}(\pi^\round, z^\round) \le 2(k+2C)\sqrt{\ln(|\cF|)T}\\
        &\Rightarrow C\alpha \sum_{\round=1}^\totalRound  \fairLoss_{\alpha}(\pi^\round, z^\round) \le 2(k+2C)\sqrt{\ln(|\cF|)T} + kT\\
        &\Rightarrow C\alpha \sum_{\round=1}^\totalRound  \fairLoss_{\alpha}(\pi^\round, z^\round) \le 6C\sqrt{\ln(|\cF|)T} + kT\\
        &\Rightarrow \sum_{\round=1}^\totalRound  \fairLoss_{\alpha}(\pi^\round, z^\round) \le \frac{6}{\alpha}\sqrt{\ln(|\cF|)T} + \frac{kT}{\alpha C}\\
        &\Rightarrow \sum_{\round=1}^\totalRound  \fairLoss_{\alpha}(\pi^\round, z^\round) \le \frac{6}{\alpha}\sqrt{\ln(|\cF|)T} + \frac{k}{\alpha}T^{\frac{3}{4}}
    \end{align*}
    where the first implication follows from the fact that \[
    \sum_{\round=1}^{\totalRound} \sum_{\subRound=1}^{\totalSubRound}\ell\left(\pi^\round \left(x^\round_{\subRound}\right), y^\round_{\subRound}\right) - \ell\left(\pi^* \left(x^\round_{\subRound}\right), y^\round_{\subRound}\right) \ge -kT.
    \]
    
    As for the misclassification regret, consider $\epsilon = 0$.
    \begin{align*}
        &\sum_{t=1}^T \misclassLoss(\pi^t, z^t) - \min_{\pi^* \in Q_{\alpha}}\sum_{t=1}^T \misclassLoss(\pi^*, z^t)\\ 
        &\le \sum_{\round=1}^{\totalRound} \Lagr_{C}(\pi^\round, z^t) - \min_{\pi^* \in \Delta \cF} \Lagr_{C}(\pi^*, z^t)\\
        &\le 2(k+2C)\sqrt{\ln(|\cF|)T} \le 6\sqrt{\ln(|\cF|)}T^{\frac{3}{4}}.
    \end{align*}
\end{proof}
We remark that $C$ is set to be exactly $\max(k, T^{\frac{1}{4}})$ so as to bound both the misclassification regret and the fairness loss with $O(T^\frac{3}{4})$, but other trade-off between the two is still possible.

\subsection{Follow-The-Perturbed-Leader}
\label{subsec:ftpl}
Running exponential weights as in Section \ref{subsec:exp-weights} in general cannot be done in an efficient manner, as such methods need to calculate the loss for each hypothesis $f \in \cF$ or for each possible labeling in the case $\cF$ has bounded VC-dimension. In this section, we aim to design an algorithm for our problem setting that is oracle-efficient. 

Specifically, we show how the algorithm proposed by \citet{SyrgkanisKS16} can be used as an $\Algorithm_{\batch}$ to achieve sublinear regret in the online batch classification setting in an oracle efficient manner. However, we remark that our approach of leveraging $\ftpl$ requires us to relax how adaptive the environment can be in terms of choosing $(x^t, y^t)$. Previously, we allowed the environment to choose $(x^t, y^t)$ with the full knowledge of the deployed policy $\pi^t$. Here, we make an assumption that $(x^t, y^t)$ can be formed as a function of $h^t=((\pi^1, z^1), \dots, (\pi^{t-1}, z^{t-1}))$ but \emph{not} $\pi^t$.

Let us now consider the setting that \citet{SyrgkanisKS16} study. They consider an adversarial contextual learning setting where in each round $\round$, the learner randomly deploys some hypothesis\footnote{They refer to this as a policy, but we say hypothesis just to be consistent with our terminology where a function that maps to $\{0,1\}$ is called hypothesis and policy is reserved for a mixture of hypotheses that maps to $[0,1]$.} $\policy^\round \in \Psi$ where $\Psi: \policyContextClass \to \{0,1\}^k$, and the environment chooses $(\policyContext^\round, w^\round) \in \policyContextClass \times \R^k$, where $k$ indicates the number of possible actions that can be taken for the context $\policyContext^\round$ whose associated loss vector is $w^k$. The only knowledge at round $\round$ \emph{not} available to the environment is the randomness that the learner uses to choose $\policy^\round$, but the environment may know the actual distribution $D^t$ over $\psi^t$ that the learner has in round $t$ just not the realization of it. At the end of the round, the learner suffers some loss $\onlineLoss(\policy^\round, (\policyContext^t, w^\round))$. 

\citet{SyrgkanisKS16} show that, in the small separator setting, they can achieve sublinear regret given that they can compute a separator set prior to learning. We first give the definition of a separator set and then state their regret guarantee. 

\begin{definition}
A set $S=(\policyContext_1, \dots, \policyContext_n)$ is called a \emph{seperator set} for $\policyClass: \policyContextClass \to \{0,1\}^k$ if for any different $\policy$ and $\policy'$ in $\policyClass$, there exists $\policyContext \in S$ such that $\policy(\policyContext) \neq \policy'(\policyContext)$.
\end{definition}

\begin{theorem}[\citet{SyrgkanisKS16}]
\label{thm:context-ftpl-regret}
For any adversarially and adaptively chosen sequence of $(\policyContext^\round, w^\round)_{\round=1}^\totalRound$, $\text{\textsc{context-ftpl}}$ initialized with a separator set $S$ and parameter $\omega$ deploys $(\psi^t)_{t=1}^T$ with the following regret: for any $\psi^* \in \Psi$,
\[
\sum_{\round=1}^{\totalRound} \E_{\psi^t \sim D^t}\left[\onlineLoss(\policy^\round, (\policyContext^\round, w^\round))\right] - \sum_{\round=1}^{\totalRound} \onlineLoss(\policy^*, (\policyContext^\round, w^\round)) \le 4\omega k n \sum_{\round=1}^\totalRound \E_{\psi^t \sim D^t}[\norm{\onlineLoss(\cdot, (\cdot, w^t))}^2_{*}] + \frac{10}{\omega} \sqrt{n k} \ln(\vert \Psi \vert),
\]
where $n = \vert S \vert$, $\norm{\onlineLoss(\cdot,(\cdot, w))}_{*} = \max_{\psi, \policyContext} |\onlineLoss(\psi, (\policyContext, w))|$ and $D^t$ is the implicit distribution over $\psi^t = \ftpl((\policyContext^\tau, w^\tau)_{\tau=1}^{t-1})$ that $\ftpl$ has in each round $t$.
\end{theorem}

Our online batch classification setting can be easily reduced to their setting by simply considering the batch of instances by setting $\policyContext^t = \bx^\round = (x^t_1, \dots, x^t_k)$, meaning we set $\policyContextClass = \cX^k$ and form the associated loss vector as $w^t_i = \frac{1-2y^t_i}{2k}$ for each $i \in [k]$. And we view each hypothesis as $\policy_{f}(x^\round) = (f(x^\round_1), \dots, f(x^\round_\totalSubRound))$. In other words, we can define the hypothesis class induced by $\cF$ as 
\[
\policyClass_{\cF, k} = \left\{\forall f \in \cF: (x_{\subRound})_{\subRound=1}^\totalSubRound \mapsto (f(x_{\subRound}))_{\subRound=1}^\totalSubRound \right\}.
\]
Note that $|\cF| = |\policyClass_{\cF, k} |$. And we can use the following linear loss
\begin{align*}
    \onlineLoss_{\batch, k}\left(\psi_f, (\policyContext^t, w^t)\right)  &= \left\langle \psi_f(\policyContext^t), w^t \right\rangle.
\end{align*}
Note that by construction, the difference in $L_{\batch, k}$ over $(\policyContext^t,w^t)$ between $\psi_f$ and $\psi_{f'}$ preserves the difference in misclassification loss over $({x'}^t,{y'}^t)$ between $f$ and $f'$:
\begin{lemma}
Write $k'=k+2C$.
\label{lem:fptl-reduction-loss}
 \[
 2k'\left(\onlineLoss_{\batch, k'}\left(\psi_f , (\policyContext^t, w^t)\right) - \onlineLoss_{\batch, k'}\left(\psi_{f'}, (\policyContext^t, w^t)\right) \right) = \sum_{\subRound=1}^{\totalSubRound'} \ell(f(x_{\subRound}), y_{\subRound}) - \sum_{\subRound=1}^{\totalSubRound'} \ell(f'({x'}^t_{\subRound}), {y'}^t_{\subRound})
 \]
\end{lemma}
\begin{proof}
\begin{align*}
&2k'\left(\onlineLoss_{\batch, k'}\left(\psi_f , (\policyContext^t, w^t)\right) - \onlineLoss_{\batch, k'}\left(\psi_f , (\policyContext^t, w^t)\right) \right) \\
&= \left\langle (f({x'}^t_\subRound))_{\subRound=1}^\totalSubRound, 1-2{y'}^t \right\rangle - \left\langle (f'({x'}_\subRound))_{\subRound=1}^\totalSubRound, 1-2{y'}^t \right\rangle\\ 
&= \left(\sum_{\subRound=1}^{k'} (1-f({x'}^t_\subRound)) \cdot {y'}^t_\subRound  + f({x'}^t_\subRound) \cdot (1-{y'}^t_\subRound)\right)- \left(\sum_{\subRound=1}^{k'} (1-f'({x'}_\subRound)) \cdot {y'}^t_\subRound  + \pi({x'}^t_\subRound) \cdot (1-y^t_\subRound)\right)\\
&=\sum_{\subRound=1}^{\totalSubRound'} \ell(f({x'}^\round_{\subRound}), {y'}^\round_{\subRound}) - \sum_{\subRound=1}^{\totalSubRound'} \ell(f'({x'}^\round_{\subRound}), {y'}^\round_{\subRound})
\end{align*}
\end{proof}
\citet{SyrgkanisKS16} assume an optimization oracle with respect to $L$
\begin{align*}
    M_L(\{(\policyContext^j, y^j)\}_{j=1}^P) = \arg\min_{\policy \in \policyClass} \sum_{j=1}^P L(\policy, (\policyContext^j, w^j))
\end{align*}
which in our case corresponds to the following oracle:
\begin{align*}
    M_{\onlineLoss_{\batch, k}}(\{(\policyContext^j, w^j)\}_{j=1}^D) = \psi_f \quad \text{where } f = \arg\min_{f \in \cF} \sum_{j=1}^D \sum_{i=1}^k f(x^j_i) w_i^j.
\end{align*}

Note that this is equivalent to a weighted empirical risk minimization oracle:
\begin{align*}
    &\arg\min_{f \in \cF} \sum_{j=1}^D \sum_{i=1}^k f(x^j_i) w_i^j \\
    &= \arg\min_{f \in \cF} \sum_{j=1}^D \sum_{i=1}^k f(x^j_i) p^j_i\left(\frac{1-2y^j_i}{2k}\right)\\
    &= \arg\min_{f \in \cF} \sum_{j=1}^D \sum_{i=1}^k  p^j_i \ell(f, (x^j_i, y^j_i))
\end{align*}
where $y^j_i = -\text{sign}(w^j_i)$ and $p^j_i = \frac{w^t_i}{y^t_i}$ for each $j \in [D], i \in [k]$. We remark that not all $w^j$ that we feed to the oracle will be of the form $\{\pm \frac{1}{2k}\}$ and $p^j_i = 1$ because $\ftpl$ requires calling the oracle not just on the set of $\policyContext^t, w^t$ that we create from $x^t$ and $y^t$ --- for stability reasons, it also adds in contexts from the separator set and associate each of those contexts with a random vector where each coordinate is drawn from the Laplace distribution. 

Furthermore, we can turn any separator set $S \subseteq \cX$ for $\cF$ into an equal size separator set $S' \subseteq \policyContextClass$ for $\policyClass$. In fact, the construction is as follows:
\[
S' = \{ \forall x \in S: \policyContext_x = (x, \arbitraryInstance, \dots, \arbitraryInstance)\},
\]
where $\arbitraryInstance$ is some arbitrary instance in $\cX$.

\begin{lemma}
If $S$ is the separator set for $\cF$, then $S'$ is a separator set for $\policyClass_\cF$.
\end{lemma}
\begin{proof}
Fix any $f$ and $f'$ where $f \neq f'$. Note that by definition of $S$, there exists $x \in S$ such that $f(x) \neq f'(x)$. As a result, $\policy_{f}(\policyContext_x) \neq \policy_{f'}(\policyContext_x)$ as $(f(x), q, \dots, q) \neq (f(x'), q, \dots, q)$.
\end{proof}

Because the loss we use is linear, we take a slightly different view on the interaction between the learner and the environment. Instead of the learner sampling a hypothesis $\psi_f^t$ and having the no-regret guarantee in expectation over the randomness of sampling the hypothesis, we imagine the learner playing the actual distribution over $\psi_f^t$ it has at round $t$. We note this distribution over $\psi_f^t$ as $D^t$ and write the loss experienced by deploying a \emph{policy} $D^t$ as 
\[
    \onlineLoss_{\batch, k'}(D^t, (\policyContext^t, w^t)) = \E_{\psi^t_f \sim D^t}[\langle \psi_f^t(\policyContext), w^t\rangle] = \langle \E_{\psi^t_f \sim D^t}[\psi_f^t(\policyContext)], w^t\rangle  = \E_{\psi^t_f \sim D^t}[\onlineLoss_{\batch, k'}(\psi^t_f, (\policyContext^t, w^t)].
\]
However, $\ftpl$ never explicitly keeps track of the distribution $D^t$ but only allows a way to sample from this distribution. Therefore, we form an empirical distribution $\tilde{D}^t$ by calling into $\ftpl$ $E$ many times to approximate $D^t$ --- i.e. we write $\tilde{D}^t$ to denote the uniform distribution over $\{\psi^t_{f^t_1}, \dots, \psi^t_{f^t_1}\}$ where $\psi^t_{f^t_j}$ is the result of our $j$th call to $\ftpl$ in round $t$. We describe the overall reduction to $\ftpl$ more formally in Algorithm \ref{alg:ftpl-reduction}. 

\begin{algorithm}[H]
\SetAlgoLined
 \SetKwInput{KwInput}{Parameters}
 \KwInput{$C, \omega, E$}
  \SetKwInput{KwInput}{Input}
 \KwInput{Separating set $S$}
 Create $S' = \{ \forall x \in S: \policyContext_x = (x, \arbitraryInstance, \dots, \arbitraryInstance)\}$ where $v \in \cX$ is chosen arbitrarily.\\
 Initialize $\ftpl$ with $S'$ and $\omega$.\\
\For{$t=1, \dots, T$}{
    $\pi^t$ is deployed.\\
    Environment, \emph{without} the knowledge of $\pi^t$, chooses $(x^t, y^t)$.\\
    Auditor chooses $\rho^t_{\Auditor_\alpha} = \Auditor_\alpha(x^t, \pi^t).$
        
    \tcp{Incur misclassification and fairness loss}
    Incur misclassification loss $\sum_{k=1}^k \ell(\pi^t(x^t), y^t)$\\
    Incur fairness loss $\fairLoss(\pi^t, \rho^t_{\Auditor_\alpha})$.
        
    \tcp{Reduction to online batch classification setting and to $\ftpl$'s setting}
    $({x'}^t, {y'}^t)= R_C((x^t, y^t), \rho^t_{\Auditor_\alpha})$.\\
    $\policyContext^{t} = {x'}^t$ and $w^t_i = \frac{1-2{y'}^t_i}{2(k+2C)}$ for each $i \in [k + 2C]$.\\
    Update history $h^{t+1} = \{(\policyContext^\tau, w^\tau)\}_{\tau=1}^t$.\\
    \For{$j \in [E]$}{
        $\policy^{t+1,j}_{f} = \ftpl(h^{t+1})$.\\
        Set $f^{t+1}_j = f$ from $\policy^{t+1,j}_{f}$. 
    }
    $\pi^{t+1}$ be a uniform distribution over $\{f^{t+1}_1, \dots, f^{t+1}_E\}$.
   }
  \caption{Reduction to $\ftpl$}\label{alg:ftpl-reduction} 
\end{algorithm}

Unlike before, we have the environment choose $(x^t, y^t)$ without the knowledge of $\pi^t$. This is so that the randomness used to form $\pi^t$ by running $\ftpl$ multiple times is not revealed to the environment. If the randomness is revealed, it is possible that the auditor can take advantage of the direction in which the empirical distribution that we deploy is off from the distribution maintained by $\ftpl$ --- more specifically, we would have to take a union bound over all possible $x^t$ and $y^t$ that the environment can choose after the environment knows how $\pi^t$ has been chosen, which we can't do because there are infinitely many possible $(x,y)$.

Instead, we could argue about the concentration of the policy $\pi^t$ itself to its expected value instead of over the loss first and use the concentration over the distribution of hypotheses to argue for the concentration of the loss. However, the loss needs to average over each hypothesis or each possible labeling in the case of bounded VC-dimension, so our estimation error in the distribution over each hypothesis will add up over each $f \in \cF$, resulting in linear dependence on $|\cF|$, or over each possible labeling induced by $\cF$ incurring estimation error linear in $O(k^V)$ where $V$ is the VC-dimension of $\cF$.

Hence, by hiding the randomness used to sample the empirical distribution from the environment (i.e. $(x^t, y^t)$ has to be chosen without access to $\tilde{D}^t$), the only thing in the loss $\onlineLoss_{\batch, k+2C}$ that is adaptive to the deployed policy is the auditor. However, the auditor only has $k^2+1$ options (i.e. choose a pair out of $k^2$ pair or output null), so we can easily union bound over these options.

\begin{lemma}
With probability $1-\delta$ over the randomness of $\tilde{D}^t$ (i.e. sampling $f^t_j$ for $j \in [E]$), we have
\begin{align*}
    \left| \onlineLoss_{\batch, k+2C}(\tilde{D}^t, (\policyContext^t, w^t)) - \onlineLoss_{\batch, k+2C}(D^t, (\policyContext^t, w^t)) \right| \le \sqrt{\frac{\ln(\frac{2T (k^2+1)}{\delta})}{2E}}
\end{align*}
for every round $t \in [T]$ where $\psi^t = \ftpl((\policyContext^\tau, w^\tau)_{\tau=1}^{t-1})$ is distributed according $D^t$. $\tilde{D}^t$ is the uniform distribution over $\{\psi^{t,j}_{f}\}_{j \in [E]}$. $(\policyContext^t, w^t)$ is determined according to $(x^t, y^t)$ and $\rho_{\Auditor_\alpha}^t = \Auditor(x^t, \pi^t)$ where $\pi^t$ is the corresponding policy for $\tilde{D}^t$ that is deployed in round $t$ as shown in Algorithm \ref{alg:ftpl-reduction}.
\end{lemma}
\begin{proof}
Fix the round $t \in [T]$. Note that
\begin{align*}
    &\onlineLoss_{\batch, k+2C}(\tilde{D}^t, (\policyContext^t, w^t)) - \onlineLoss_{\batch, k+2C}(D^t, (\policyContext^t, w^t)) \\
    &=\left\langle \E_{\psi^t_f \sim \tilde{D}^t}[\psi^t_f(\policyContext^t)],w^t\right\rangle - \left\langle \E_{\psi^t_f\sim D^t}[\psi^t_f(\policyContext^t)] ,w^t\right\rangle \\
    &= \frac{1}{E} \sum_{j \in [E]}  \sum_{i=1}^{k+2C} f^t_j({x'}^t_i) w^t_i -  \E_{\psi^t_f \sim D^t}\left[\sum_{i=1}^{k+2C}f({x'}^t_i) w^t_i\right] \\
    &= \left\langle \E_{\psi^t_f \sim \tilde{D}^t}[\psi^t_f(\policyContext^t)],w^t\right\rangle - \left\langle \E_{\psi^t_f\sim D^t}[\psi^t_f(\policyContext^t)] ,w^t\right\rangle \\
    &= \frac{1}{E} \sum_{j \in [E]}  \sum_{i=1}^{k} f^t_j(x^t_i) w^t_i + C \left(f^t_j({x'}^t_{k+1}) w^t_{k+1} + f^t_j({x'}^t_{k+C+1}) w^t_{k+C+1}\right) \\
    &-\E_{\psi_f \sim D^t}\left[\sum_{i=1}^{k}f(x^t_i) w^t_i + C \left(f(x^t_{\rho^t_1}) w^t_{k+1} + f(x^t_{\rho^t_2}) w^t_{k+C+1}\right)\right]. 
\end{align*}
Note that just by the construction of $\policyContext^t$ and ${x'}^t, {y'}^t$, we know that there are only $k^2+1$ possible options for $({x'}^t_{k+1}, {x'}_{k+C+1})$. More specifically, if $\rho^t_{\Auditor_{\alpha}} \neq \nullPair$, then ${x'}^t_{k+1} = x^t_{\rho^t_1}$ and ${x'}_{k+C+1} = x^t_{\rho^t_2}$  where $\rho^t_1 \in [k]$ and $\rho^t_2 \in [k]$. If $\rho^t_{\Auditor_\alpha} = \nullPair$, then $({x'}^t_{k+1}, {x'}_{k+C+1}) = (v,v)$ always. Furthermore, by construction, we have $w^t_{k+1} = \frac{-1}{2(k+2C)}$ and $w^t_{k+C+1} = \frac{1}{2(k+2C)}$ always.

Write \begin{align*}
    V_j &= \sum_{i=1}^{k} f^t_j(x^t_i) w^t_i + C \left(f^t_j(x^t_{\rho^t_1}) w^t_{k+1} + f^t_j(x^t_{\rho^t_2}) w^t_{k+C+1}\right)\\
    V &= \E_{\psi^t_f \sim D^t}\left[\sum_{i=1}^{k}f(x^t_i) w^t_i + C \left(f(x^t_{\rho^t_1}) w^t_{k+1} + f(x^t_{\rho^t_2}) w^t_{k+C+1}\right)\right]
\end{align*}

Note that $\E[V_j] = V$ for each $j \in [E]$. Therefore, union bounding over all possible $\rho^t \in [k]^2 \cup \{\nullPair\}$ with Chernoff bound, we have
\[
    \Pr_{(f^t_j)_{j \in [E]}}\left(\left|  \frac{1}{E}\sum_{j \in [E]}V_j - V\right| \ge \sqrt{\frac{\ln(\frac{2(k^2+1)}{\delta})}{2E}}  \right) \le \delta.
\]

Union bounding over all round $t \in [T]$, we have with probability $1-\delta$, 
\begin{align*}
    \left|\onlineLoss_{\batch, k+2C}(\tilde{D}^t, (\policyContext^t, w^t)) - \onlineLoss_{\batch, k+2C}(D^t, (\policyContext^t, w^t))\right| \le \sqrt{\frac{\ln(\frac{2T (k^2+1)}{\delta})}{2E}}.
\end{align*}
\end{proof}

Now, we can combine all the arguments we have developed so far in order to prove that Algorithm \ref{alg:ftpl-reduction} achieves sublinear fairness loss and misclassification regret:
\begin{theorem}
\label{thm:online-batch-regret-ftpl}
Set $C = \max(k, T^\frac{2}{9}), E=T,$ and $\omega = n^{\frac{-1}{4}}{k'}^\frac{-3}{4}T^\frac{-1}{2} \ln(|\cF|)^\frac{1}{2}$ where $n$ is the size of the separator set $S$. Algorithm \ref{alg:ftpl-reduction} guarantees that with probability $1-\delta$, the following holds true:
\begin{align*}
\sum_{\round=1}^\totalRound  \fairLoss_{\alpha}(\pi^\round, z^\round) &\le \frac{1}{\alpha} O \left(n^\frac{3}{4} \ln(|\cF|)^\frac{1}{2}T^\frac{5}{9} + \sqrt{T\ln\left(\frac{Tk}{\delta}\right)}  + k T^{\frac{7}{9}}\right)\\
\sum_{t=1}^T \misclassLoss(\pi^t, z^t) - \min_{\pi^* \in Q_{\alpha}}\sum_{t=1}^T \misclassLoss(\pi^*, z^t) &\le O\left(n^\frac{3}{4} \ln(|\cF|)^\frac{1}{2}T^{\frac{7}{9}} + \sqrt{\ln\left(\frac{Tk}{\delta}\right)} T^{\frac{13}{18}}\right)
\end{align*}
    where $z^t = ((x^t, y^t), \rho^t_{O_{\alpha}})$ and $\rho^t_{O_{\alpha}}=\Auditor_{\alpha}(x^t, \pi^t)$. In other words, the misclassification regret and fairness loss is both bounded by $O(T^{\frac{7}{9}})$ with high probability.
\end{theorem}
\begin{proof}
\begin{align*}
    &\sum_{\round=1}^\totalRound \Lagr_{C,\alpha}(\pi^t, z^\round) - \min_{\pi^* \in Q_{\alpha-\epsilon}}\sum_{\round=1}^\totalRound \Lagr_{C,\alpha}(\pi^*, z^\round)\\
    &\le\sum_{\round=1}^\totalRound \Lagr_{C,\alpha}(\pi^t, z^\round) - \min_{\pi^* \in \Delta \cF}\sum_{\round=1}^\totalRound \Lagr_{C,\alpha}(\pi^*, z^\round) \\
    &=\sum_{t=1}^T \sum_{i=1}^{k+2C} \ell(\pi^t({x'}^t), {y'}^t_i) - \min_{\pi^* \in \Delta\cF}\sum_{t=1}^T \sum_{i=1}^{k+2C} \ell(\pi^*({x'}^t), {y'}^t_i) && \text{Lemma \ref{lem:reduction_regret}}\\
    &=\sum_{t=1}^T \sum_{i=1}^{k+2C} \ell(\pi^t({x'}^t), {y'}^t_i) - \min_{f^* \in \cF}\sum_{t=1}^T \sum_{i=1}^{k+2C} \ell(f^*({x'}^t), {y'}^t_i) &&
    \parbox[t]{0.4\textwidth}{
                    Optimal solution over linear objective
                    must happen at the support}
\end{align*}

Then, applying Lemma \ref{lem:fptl-reduction-loss} yields with probability $1-\delta$,
\begin{align*}
    &=(k+2C)\left(\sum_{\round=1}^{\totalRound} \onlineLoss_{\batch, k+2C}(\tilde{D}^t, (\policyContext^\round, w^\round)) - \min_{\psi^* \in \Psi_{\batch, k+2C}}\sum_{\round=1}^{\totalRound} \onlineLoss_{\batch, k+2C}(\policy^*, (\policyContext^\round, w^\round))\right)\\
    &\le(k+2C)\Bigg(\sum_{\round=1}^{\totalRound} \onlineLoss_{\batch, k+2C}(D^t, (\policyContext^\round, w^\round)) - \min_{\psi^* \in \Psi_{\batch, k+2C}}\sum_{\round=1}^{\totalRound} \onlineLoss_{\batch, k+2C}(\policy^*, (\policyContext^\round, w^\round))\\
    &+ T\sqrt{\frac{\ln(\frac{2T (k^2+1)}{\delta})}{2E}} \Bigg)
\end{align*}
Notice that $|\onlineLoss_{\batch,k+2C}(\psi, (\policyContext, w^t))| \in [-\frac{1}{2},\frac{1}{2}]$ for any $\psi, x$ and $t \in [T]$ because by construction $w^t_i \in \{ \pm \frac{1}{2(k+2C)} \}$. In other words, we have $\E_{\psi \sim D^t}\left[||\onlineLoss_{\batch, k+2C}( \psi, (\cdot, w^t))||^2_{*}\right] \le \frac{1}{4}$. Theorem \ref{thm:context-ftpl-regret} gives us that a sequence of distribution $(D^t)_{t=1}^T$ achieves the following, where $D^t$ is equivalent to the distribution of $\ftpl(((\policyContext^\tau, w^\tau))_{\tau=1}^{t-1})$:
\begin{align*}
    &\sum_{\round=1}^{\totalRound} \onlineLoss_{\batch, k+2C}(D^t, (\policyContext^\round, w^\round)) - \min_{\psi^* \in \Psi_{\batch, k+2C}}\sum_{\round=1}^{\totalRound} \onlineLoss_{\batch, k+2C}(\policy^*, (\policyContext^\round, w^\round)) \\
    &\le \omega k' n T + \frac{10}{\omega} \sqrt{n k'} \ln(\vert \Psi_{\cF, k+2C} \vert).
\end{align*}

Therefore, writing $k'=k+2C$, we have with probability $1-\delta$,
\begin{align*}
&\sum_{\round=1}^\totalRound \Lagr_{C,\alpha}(\pi^t, z^\round) - \min_{\pi^* \in Q_{\alpha-\epsilon}}\sum_{\round=1}^\totalRound \Lagr_{C,\alpha}(\pi^*, z^\round)\\
&\le k' \left( \sum_{\round=1}^{\totalRound} \E_{\psi^t \sim D^t}\left[\onlineLoss_{\batch, k'}(\policy^\round, (\policyContext^\round, w^\round))\right] - \sum_{\round=1}^{\totalRound} \onlineLoss_{\batch, k'}(\policy^*, (\policyContext^\round, w^\round)) + \sqrt{T \frac{\ln(\frac{2T (k^2+1)}{\delta})}{2}}\right)\\
&\le k' \left(\omega k' n T + \frac{10}{\omega} \sqrt{n k'} \ln(\vert \Psi_{\cF, k+2C} \vert) + \sqrt{T \frac{\ln(\frac{2T (k^2+1)}{\delta})}{2}} \right).
\end{align*}

Setting $\omega = n^{\frac{-1}{4}}{k'}^\frac{-3}{4}T^\frac{-1}{2} \ln(|\cF|)^\frac{1}{2}$, we then have
\begin{align*}
    &\sum_{\round=1}^\totalRound \Lagr_{C,\alpha}(\pi^t, z^\round) - \min_{\pi^* \in Q_{\alpha-\epsilon}}\sum_{\round=1}^\totalRound \Lagr_{C,\alpha}(\pi^*, z^\round)\\
    &\le O\left(\left(n^\frac{3}{4} {k'}^{\frac{5}{4}}\ln(|\cF|)^\frac{1}{2} + k'\sqrt{ \ln\left(\frac{Tk}{\delta}\right)} \right)T^{\frac{1}{2}}  \right)
\end{align*}

With the same argument as in the proof of Theorem \ref{thm:exp-guarantee}, 
we get that for $\epsilon=\alpha$ 
\begin{align*}
        &C\alpha \sum_{\round=1}^\totalRound  \fairLoss_{\alpha}(\pi^\round, z^\round) \le O\left(\left(n^\frac{3}{4} {k'}^{\frac{5}{4}}\ln(|\cF|)^\frac{1}{2} + k'\sqrt{ \ln\left(\frac{Tk}{\delta}\right)} \right)T^{\frac{1}{2}}  \right) + kT\\
        &\Rightarrow C\alpha \sum_{\round=1}^\totalRound  \fairLoss_{\alpha}(\pi^\round, z^\round) \le O\left(\left(n^\frac{3}{4} {C}^{\frac{5}{4}}\ln(|\cF|)^\frac{1}{2} + C\sqrt{ \ln\left(\frac{Tk}{\delta}\right)} \right)T^{\frac{1}{2}}  \right) + kT\\
        &\Rightarrow \sum_{\round=1}^\totalRound  \fairLoss_{\alpha}(\pi^\round, z^\round) \le \frac{1}{\alpha}O\left(\left(n^\frac{3}{4} {C}^{\frac{1}{4}}\ln(|\cF|)^\frac{1}{2} + \sqrt{ \ln\left(\frac{Tk}{\delta}\right)} \right)T^{\frac{1}{2}} + \frac{kT}{C}\right) \\
        &\Rightarrow \sum_{\round=1}^\totalRound  \fairLoss_{\alpha}(\pi^\round, z^\round) \le \frac{1}{\alpha} O \left(n^\frac{3}{4} \ln(|\cF|)^\frac{1}{2}T^\frac{5}{9} + \sqrt{T\ln\left(\frac{Tk}{\delta}\right)}  + k T^{\frac{7}{9}}\right).
    \end{align*}
As for the misclassification regret, we have with $\epsilon=0$
    \begin{align*}
        &\sum_{t=1}^T \misclassLoss(\pi^t, z^t) - \min_{\pi^* \in Q_{\alpha}}\sum_{t=1}^T \misclassLoss(\pi^*, z^t) \\
        &\le O\left(\left(n^\frac{3}{4} {k'}^{\frac{5}{4}}\ln(|\cF|)^\frac{1}{2} + k'\sqrt{\ln\left(\frac{Tk}{\delta}\right)} \right)T^{\frac{1}{2}}  \right) \\
        &\le O\left(n^\frac{3}{4} \ln(|\cF|)^\frac{1}{2}T^{\frac{7}{9}} + \sqrt{\ln\left(\frac{Tk}{\delta}\right)} T^{\frac{13}{18}}\right)
    \end{align*}
\end{proof}

We only focus on their small separator set setting, although their transductive setting (i.e. the contexts $(x_t)_{t=1}^T$ are known in advance) should naturally follow as well.

\section{Generalization} \label{sec:gen}

We observe that until this point, all of our results apply to the more general setting where individuals arrive in any adversarial fashion. In order to argue about generalization, in this section, we will assume the existence of an (unknown) data distribution from which individual arrivals are drawn:
\[
\{\{(x^\round_\subRound,y^\round_\subRound)\}_{\subRound=1}^{\totalSubRound}\}_{\round=1}^{\totalRound}\sim_{i.i.d.} {\cD}^{\totalRound\totalSubRound}
\]
In spite of the data being drawn i.i.d., there are two main technical challenges in establishing  generalization guarantee: (1) the auditor's fairness feedback at each round is limited to only incorporate reported fairness violations (as opposed to stronger forms of feedback of, for example, the distances between arriving individuals, etc.), where it only entails a single fairness violation with regards to the policy deployed in that round, and (2) both the deployed policies and the auditor are adaptive over rounds. To overcome these challenges, we will draw a connection between our established regret guarantees in Section \ref{sec:noregret} and the learner's distributional accuracy and fairness guarantees. 

As formerly discussed in Section \ref{sec:introduction}, such strategies of leveraging regret guarantees to prove distributional bounds are known in the online learning literature as online to batch conversion schemes (see, e.g. \cite{CesaBianchi04,HelmboldWarmuth}). As we will see next, while accuracy generalization follows from standard such arguments, the fairness generalization bound we wish to establish does not, and will require a more elaborate technique.  

In particular, we will analyze the generalization bounds for the average policy over rounds.

\begin{definition}[Average Policy]
Let $\pi^t$ be the policy deployed by the algorithm at round $t$. The average policy $\pi^{avg}$ is defined by:
\[
\forall x: \pi^{avg}(x) = \frac{1}{T}\sum_{t=1}^T \pi^t(x).
\]
\end{definition}
We denote the set of policies which are fair on the entire space of pairs of individuals by 
\begin{align*}
    \Pi_{\alpha} &= \{\pi \in \Delta \cF: \text{$\pi$ is ($\alpha,0)$-fair}\}\\
    &= \{\pi \in \Delta \cF: |\pi(x)-\pi(x')| \le d(x,x') + \alpha\quad \forall(x,x')\in\cX^2\}.
\end{align*}
Note that $\Pi_{\alpha} \subseteq Q_\alpha$ always.

\subsection{Accuracy Generalization}
We begin by stating the accuracy generalization guarantee for the average policy.

\begin{restatable}[Accuracy Generalization]{theorem}{thmAvgLoss}
\label{thm:avgloss} 
Suppose we are given an algorithm such that for any sequence $(x^t, y^t)_{t=1}^T$, it produces $(\pi^t)_{t=1}^T$ such that with probability $1-\delta$,
\[
    \sum_{t=1}^T \misclassLoss(\pi^t, (x^t, y^t)) - \min_{\pi^* \in Q_{\alpha}} \sum_{t=1}^T \misclassLoss(\pi^*, x^t, y^t) \le \Regret_{\acc}.
\]

Then, with probabilty $1-\delta$, the misclassification loss of $\pi^{avg}$ is upper bounded by
\begin{align*}
\E_{(x,y)\sim\cD}\left[\ell(\pi^{avg}(x),y)\right] \leq  \min\limits_{\pi^* \in \Pi_{\alpha}}\E_{(x,y)\sim\cD} \left[\ell(\pi(x),y)\right] + \frac{1}{\totalSubRound\totalRound}\Regret_{\acc} + \sqrt{\frac{8\ln\left(\frac{4}{\delta}\right)}{\totalRound}}
\end{align*}
\end{restatable}

The high level proof idea for the accuracy generalization bound of Theorem \ref{thm:avgloss} is as follows. While fixing the randomness of the algorithm, we apply the Azuma's inequality over the randomness of sampling $((x^t, y^t))_{t=1}^T$ and similarly the Chernoff bound for any fixed $\pi^* \in \Pi_{\alpha}$ in order to show that our empirical misclassification regret is concentrated around the true misclassification regret over the distribution $\cD$. Separately, for any fixed $((x^t, y^t))_{t=1}^T$, we know that over the randomness of the algorithm, the misclassification regret must be bounded by $\Regret_{\acc}$. Then, we can union bound to argue that our true misclassification regret must be concentrated around $\Regret_{\acc}$. The full proof is given in Appendix \ref{app:gen}.

\subsection{Fairness Generalization}

A more challenging task is, however, to argue about the fairness generalization guarantee of the average policy. To provide some intuition for why this is the case, let us first attempt to upper bound the probability of running into an $\alpha$-fairness violation by the average policy on a randomly selected pair of individuals:
\begin{observation} \label{obs:basic}
Suppose for all $\round$, $\pi^\round$ is $(\alpha,\beta^\round)$-fair. Then, $\pi^{avg}$ is $\left(\alpha,\sum\limits_{\round=1}^\totalRound \beta^\round\right)$-fair.
\end{observation}

This bound is very dissatisfying, as the statement is vacuous when $\sum_{\round=1}^\totalRound \beta^\round \geq 1$. The reason for such a weak guarantee is that by aiming to upper bound the unfairness probability for the original fairness violation threshold $\alpha$, we are subject to worst-case compositional guarantees\footnote{This is the case where, for every pair of individuals $x,x'$ on which there exists a policy in $\{\pi^1,\dots,\pi^\totalRound\}$ that has an $\alpha$-fairness violation: (1) every policy $\pi\in\{\pi^1,\dots,\pi^\totalRound\}$ that has an $\alpha$-fairness violation on $x,x'$ has a maximal violation (of value 1), (2) all non-violating policies (in $\{\pi^1,\dots,\pi^\totalRound\}$) on $x,x'$ are arbitrarily close to the violation threshold $\alpha$ on this pair, and (3) all of the directions of the policies' predictions' differences on $x,x'$ are correlated (no cancellations when averaging over policies).} in the sense that the average policy may result to have an $\alpha$-fairness violation on any fraction of the distribution (over pairs) where one or more of the deployed policies induces an $\alpha$-fairness violation.
This bound is tight, and we refer the reader to Appendix \ref{app:gen} for the full example.

\paragraph{Interpolating $\alpha$ and $\beta$}
To circumvent the above setback, our strategy will be to relax the target violation threshold of the desired fairness guarantee of the average policy to $\alpha' > \alpha$. How big should we set $\alpha'$? A good intuition may arrive from considering the following thought experiment: assume worst-case compositional guarantees, and then, select a pair of individuals $x,x'$ on which the average policy has an $\alpha'$-fairness violation. We aim to lower bound the number of policies from $\{\pi^1,\dots,\pi^\totalRound\}$ that have an $\alpha$-fairness violation on this pair. As we will see, setting $\alpha'$ to be sufficiently larger will force the number of these policies required to produce an $\alpha'$-fairness violation of the average policy on $x,x'$ to be high, resulting in the following improved bound: 

\begin{restatable}{lem}{lemGen}
\label{lem:gen}
Suppose that for all $\round$, $\pi^\round$ is $(\alpha,\beta^\round)$-fair $(0\leq\beta^\round\leq 1)$. For any integer $\interpolation \leq T$, $\pi^{avg}$ is $\left(\alpha + \frac{\interpolation}{\totalRound},\frac{1}{\interpolation}\sum\limits_{\round=1}^\totalRound \beta^\round\right)$-fair.
\end{restatable}

\paragraph{High-Level Proof Idea} Setting $\alpha'=\alpha+\frac{\interpolation}{\totalRound}$ has the following implication: for any pair of individuals $(x,x')$, in order for $\pi^{avg}$ to have an $\alpha'$-fairness violation on $x,x'$, at least $\interpolation$ of the policies in $\{\pi^1,\dots,\pi^\totalRound\}$ must have an $\alpha$-fairness violation on $x,x'$. We will then say a subset $A\subseteq\cX\times\cX$ is $\alpha$-covered by a policy $\pi$, if $\pi$ has an $\alpha$-violation on every element in $A$. We will denote by $A_\interpolation^{\alpha} \subseteq \cX\times\cX$ the subset of pairs of elements from $\cX$ that are $\alpha$-covered by at least $\interpolation$ policies in $\{\pi^1,\dots,\pi^\totalRound\}$. Next, consider the probability space $\cD\vert_{\cX}\times\cD\vert_{\cX}$ over pairs of individuals. The lemma then follows from observing that for any $\interpolation\leq\totalRound$, $\Pr(A_\interpolation^{\alpha}) \leq \frac{1}{\interpolation} \Pr(A_1^{\alpha})$, as this will allow us to upper bound the probability of an $\alpha'$-fairness violation by the stated bound.

\begin{proof}
Fix $\{\pi^\round\}_{\round=1}^\totalRound$, and assume that $\forall \round: \pi^\round$ is $(\alpha,\beta^\round)$-fair. If we set $\interpolation\leq \totalRound$, then we know that
\begin{align}
&\Pr_{x,x'}\left[\left|\pi^{avg}(x)-\pi^{avg}(x')\right|- \metric(x,x')> \alpha+\frac{\interpolation}{\totalRound}\right]\notag \\
&\leq \Pr_{x,x'}\left[\exists \{i_1,\dots,i_\interpolation\} \subseteq [\totalRound],\forall j ,\vert \pi^{i_j}(x)-\pi^{i_j}(x')\vert - \metric(x,x') > \alpha\right]\label{tra:lower} \\
&\leq \frac{1}{\interpolation}\sum\limits_{\round=1}^\totalRound \Pr_{x,x'}\left[\vert\pi^\round(x)-\pi^\round(x')\vert-\metric(x,x')>\alpha\right]\label{tra:covering}\\
&\leq \frac{1}{\interpolation}\sum\limits_{\round=1}^\totalRound \beta^\round\notag
\end{align}

Transition \eqref{tra:lower} is given by the following observation. Fix any $x,x'$ and assume 
\begin{align*}
\left\vert\{\pi^\round:\round\in[\totalRound],\vert\pi^\round(x)-\pi^\round(x')\vert-\metric(x,x')>\alpha\}\right\vert < \interpolation.
\end{align*}
Then, we have 
\begin{align*}
\vert\pi^{avg}(x)-\pi^{avg}(x')\vert-\metric(x,x')&\le \frac{1}{T}\sum_{t=1}^T|\pi^t(x) - \pi^t(x')|-\metric(x,x')\\
&\leq\frac{\interpolation+(\totalRound-\interpolation)\alpha}{\totalRound}\\
&=\alpha+\frac{\interpolation}{\totalRound}-\frac{\alpha \interpolation}{\totalRound} \\
&< \alpha +\frac{\interpolation}{\totalRound}. 
\end{align*}
In other words, we must have that
\[
    \vert\pi^{avg}(x)-\pi^{avg}(x')\vert-\metric(x,x')>  \alpha +\frac{\interpolation}{\totalRound} 
\]
implies
\[
    \left\vert\{\pi^\round:\round\in[\totalRound],\vert\pi^\round(x)-\pi^\round(x')\vert-\metric(x,x')>\alpha\}\right\vert \ge \interpolation.
\]

Transition \eqref{tra:covering} stems from the following argument: for any $x,x'$, denote by \[
V_{x,x'}^{\alpha} := \{t\in[T],\vert\pi^t(x)-\pi^t(x')\vert-\metric(x,x')> \alpha\}
\] 
the rounds where the deployed policy would have an $\alpha$-fairness violation on $x,x'$. We know that
\begin{align*}
&\frac{1}{\interpolation}\sum\limits_{\round=1}^\totalRound \Pr_{x,x'}\left[\vert\pi^\round(x)-\pi^\round(x')\vert-\metric(x,x')>\alpha\right] \\
&= \frac{1}{\interpolation}\sum_{t=1}^\totalRound\int_{x,x'}\ind[\vert \pi^t(x)-\pi^t(x')\vert -\metric(x,x')> \alpha]\cD_{\cX^2}(x,x') \\
&= \frac{1}{\interpolation}\int_{x,x'}\sum_{t=1}^\totalRound\ind[\vert \pi^t(x)-\pi^t(x')\vert -\metric(x,x')> \alpha]\cD_{\cX^2}(x,x')\\
&= \frac{1}{\interpolation} \int_{x,x'}|V_{x,x'}^{\alpha}|\cD_{\cX^2}(x,x') \\
&\geq \int_{x,x'} \ind\left[\vert V_{x,x'}^{\alpha} \vert \geq \interpolation\right]\cD_{\cX^2}(x,x') \\
&= \Pr_{x,x'}\left[\exists \{i_1,\dots,i_\interpolation\} \subseteq [\totalRound],\forall j ,\vert \pi^{i_j}(x)-\pi^{i_j}(x')\vert - \metric(x,x') > \alpha\right],
\end{align*}
where $\cD_{\cX^2}(x,x')$ denotes the probability measure of $x,x'$ defined by $\cD \vert_\cX \times \cD \vert_\cX$. 
\end{proof}

Equipped with Lemma \ref{lem:gen}, we next note that the stochastic assumption allows us to link the distributional fairness guarantees of the deployed policies to the algorithm's regret. The implication of this, which we state next, is that the case of deploying too many highly unfair policies along the interaction must result in a contradiction to the algorithm's proven regret guarantee. Hence, we will be able to bound the sum of the distributional $\alpha$-fairness guarantees of all the policies deployed by the algorithm. 

\begin{restatable}{lem}{lembBoundedsum}
\label{lem:boundedsum}
With probability $1-\delta$, we have
\[ \sum_{\round=1}^\totalRound \beta^\round \leq \Regret_{\fair} + \sqrt{2\totalRound \ln\left(\frac{2}{\delta}\right)}\]
\end{restatable}
\begin{proof}
We first prove the following concentration inequality.
\begin{lemma}\label{lem:azuma_fairnes} For any fixed randomness of the algorithm,
 \[
 \Pr_{(\bx^\round, \by^\round)_{\round=1}^{\totalRound}}\left[\left\vert \sum_{\round=1}^\totalRound 
        \ind\left(\Auditor_{\alpha}\left({\bx}^{\round}, \pi^\round\right) \neq \emptyset \right) - \E_{({\bx}^{\prime\round}, {\by}^{\prime\round})_{\round=1}^{\totalRound}}\left[\sum_{\round=1}^\totalRound 
        \ind\left(\Auditor_{\alpha}\left({\bx}^{\prime\round}, \pi^\round\right) \neq \emptyset \right)\right]\right\vert \ge \gamma \right] \le 2\exp\left(-\frac{\gamma^2}{2\totalRound}\right)
 \]
\end{lemma}
\begin{proof}
Consider the following sequence $(B^\round)_{\round=1}^\totalRound$:
\[
B^\round = \sum_{j=1}^\round 
        \ind\left(\Auditor_{\alpha}\left(\bx^j, \pi^j\right) \neq \emptyset \right) - \E_{({\bx}^{\prime j},\by^{\prime j})_{j=1}^{\round}}\left[\sum_{j=1}^\round 
        \ind\left(\Auditor_{\alpha}\left({\bx}^{\prime j}, \pi^j\right) \neq \emptyset \right)\right].
\]
Note that this is a martingale: $\E[B^\round | B^1, \dots, B^{\round-1}] = B^{\round-1}$ because $\pi^t$ is fixed. Now, we apply Azuma's inequality. Since $\vert B^\round - B^{\round-1}\vert \le 1$, we have
\[
\Pr\left[\vert B^\totalRound - B^1 \vert  \ge \gamma \right] \le 2\exp\left(\frac{-\gamma^2}{2\totalRound}\right).
\]
\end{proof}

\begin{lemma}
\label{lem:sum_beta_lowerbound}
Fix the sequence of policies $(\pi^\round)_{\round=1}^\totalRound$. Suppose that each $\pi^\round$ is $(\alpha,\beta^\round)$-fair. Then:
\[
\sum_{\round=1}^\totalRound \E_{({\bx}^{'\round}, {\by}^{'\round})_{\round=1}^\totalRound}
\left[
        \ind\left(\Auditor_{\alpha}\left({\bx'}^{\round}, \pi^\round\right) \neq \nullPair \right)
\right] \ge \sum_{\round=1}^\totalRound \beta^\round.
\]
\end{lemma}
\begin{proof}
We will lower bound the probability of having an $\alpha$-fairness violation on a pair of individuals among those who have arrived in a single round. Because all possible pairs within a single batch of independently arrived individuals are not independent, we resort to a weaker lower bound by dividing all of the arrivals into independent pairs. As we go through the batch, we take every two individuals to form a pair, and note that these pairs must be independent.

\begin{align*}
&\E_{({\bx}^{\prime\round}, {\by}^{\prime\round})_{\round=1}^\totalRound}
\left[
        \ind\left(\Auditor_{\alpha}\left({\bx}^{\prime\round}, \pi^\round\right) \neq \emptyset \right)
\right]\\
&=\Pr_{({\bx}^{\prime\round}, {\by}^{\prime\round})_{\round=1}^\totalRound}\left[\exists \subRound,\subRound' \in [\totalSubRound] : \vert\pi^\round({x}^{\prime\round}_{\subRound})-\pi^\round({x}^{\prime\round}_{\subRound'})\vert - \metric({x}^{\prime\round}_{\subRound}, {x}^{\prime\round}_{\subRound'}) > \alpha\right]\\
&\geq \Pr_{({x'}^t,{y'}^t)_{t=1}^T}\left[\exists i \in \left[\floor*{\frac{\totalSubRound}{2}}\right] : \vert\pi^\round({x}^{\prime\round}_{2i-1})-\pi^t({x}^{\prime\round}_{2i})\vert - \metric({x}^{\prime\round}_{2i-1},{x}^{\prime\round}_{2i}) > \alpha\right]\\
&= 1 - \Pr_{({x'}^t,{y'}^t)_{t=1}^T}\left[\forall i \in \left[\floor*{\frac{\totalSubRound}{2}}\right] : \vert\pi^t({x}^{\prime\round}_{2i-1})-\pi^\round({x}^{\prime\round}_{2i})\vert - \metric({x}^{\prime\round}_{2i-1},{x}^{\prime\round}_{2i}) \leq \alpha\right]\\
&= 1 - \prod_{i=1}^{\floor*{\frac{\totalSubRound}{2}}} (1-\beta^\round)\\
&= 1-(1-\beta^\round)^{\floor*{\frac{\totalSubRound}{2}}}\\
&\ge \beta^\round
\end{align*}

\end{proof}

Next, we combine Lemma \ref{lem:azuma_fairnes} and Lemma \ref{lem:sum_beta_lowerbound}. With probability $1-\delta$ over the randomness of $((x^t, y^t))_{t=1}^T$, for any fixed randomness of the algorithm, we have
\begin{align}
\sum_{\round=1}^\totalRound 
        \ind\left(\Auditor_{\alpha}\left({\bx}^{\round}, \pi^\round\right) \neq \nullPair \right) &\ge \sum_{\round=1}^\totalRound \E_{({\bx}^{\prime\round}, {\by}^{\prime\round})_{\round=1}^\totalRound}\left[\ind\left(\Auditor_{\alpha}\left({\bx}^{\prime\round}, \pi^\round\right) \neq \nullPair \right)\right] - \sqrt{2\totalRound \ln\left(\frac{\delta}{2}\right)}\\
&\ge \sum_{\round=1}^\totalRound \beta^\round - \sqrt{2\totalRound \ln\left(\frac{\delta}{2}\right)}.
\end{align}

Suppose we write $\seed$ to denote the random seed used for the algorithm. For simplicity, write $E_1[((x^t, y^t))_{t=1}^T, \seed]$ to denote the event that the following equality holds
\[
    \sum_{\round=1}^\totalRound 
        \ind\left(\Auditor_{\alpha}\left({\bx}^{\round}, \pi^\round\right) \neq \nullPair \right) < \sum_{\round=1}^\totalRound \beta^\round - \sqrt{2\totalRound \ln\left(\frac{\delta}{2}\right)}. 
\]
And we write $E_2[((x^t, y^t))_{t=1}^T, \seed]$ to denote the event that \[
    \sum_{t=1}^T \ind\left(\Auditor_{\alpha}\left({\bx}^{\round}, \pi^\round\right) \neq \nullPair \right) > \Regret_{\fair}.
\]

We have that 
\begin{align*}
    &\forall \seed: \Pr_{(x^t, y^t)}\left[E_1\left[\left((x^t, y^t)\right)_{t=1}^T, \seed\right] | \seed\right] \le \delta\\
    &\forall ((x^t, y^t))_{t=1}^T: \Pr_{\seed}\left[E_2\left[\left(((x^t, y^t)\right)_{t=1}^T, \seed\right] | \left((x^t, y^t)\right)_{t=1}^T\right] \le \delta
\end{align*}

Therefore, we can use union bound to show
\begin{align*}
    &\Pr_{\seed, ((x^t, y^t))_{t=1}^T}\left[E_1\left[\left((x^t, y^t)\right)_{t=1}^T, \seed\right]\quad\text{or}\quad E_2\left[\left((x^t, y^t)\right)_{t=1}^T, \seed\right]\right]\\
    &\le \Pr_{\seed, ((x^t, y^t))_{t=1}^T}\left[E_1\left[\left((x^t, y^t)\right)_{t=1}^T, \seed\right]\right] + \Pr_{\seed, ((x^t, y^t))_{t=1}^T}\left[E_2\left[\left((x^t, y^t)\right)_{t=1}^T, \seed\right]\right]\\
    &\le \int_{((x^t, y^t))_{t=1}^T} \Pr_{\seed}\left[E_1\left[\left((x^t, y^t)\right)_{t=1}^T, \seed\right] \middle\vert ((x^t, y^t))_{t=1}^T\right]d\cD^{kT}(((x^t, y^t))_{t=1}^T) \\
    &+ \int_{\seed} \Pr_{((x^t, y^t))_{t=1}^T}\left[E_2\left[\left((x^t, y^t)\right)_{t=1}^T, \seed\right] \middle\vert \seed\right]d\cP(\seed)\\
    &\le 2\delta.
\end{align*}
In other words, we have that with probability $1-2\delta$, 
\begin{align*}
    \Regret_{\fair} \ge \sum_{\round=1}^\totalRound \beta^\round - \sqrt{2\totalRound \ln\left(\frac{\delta}{2}\right)}.
\end{align*}
\end{proof}

Finally, we combine the above lemmas to state our fairness generalization guarantee.

\begin{restatable}[Fairness Generalization]{theorem}{thmGen}
\label{thm:gen}
Suppose we are given an algorithm such that for any sequence $(x^t, y^t)_{t=1}^T$, it produces $(\pi^t)_{t=1}^T$ such that with probability $1-\delta$,
\[
    \sum_{t=1}^T \ind\left(\Auditor_{\alpha}\left({\bx}^{\round}, \pi^\round\right) \neq \nullPair \right)  = \sum_{t=1}^T \fairLoss_{\alpha}(\pi^t, (x^t, y^t, \Auditor_\alpha(\pi^t)))  \le \Regret_{\fair}.
\]
Then, with probability $1-\delta$, for any integer $\interpolation \leq T$, $\pi^{avg}$ is $(\alpha + \frac{\interpolation}{\totalRound}, \beta^*)$-fair where 
\[
\beta^* = \frac{1}{\interpolation}\left(\Regret_{\fair} + \sqrt{2\totalRound \ln\left(\frac{4}{\delta}\right)}\right).
\]
\end{restatable}

\begin{proof}
The theorem follows immediately by combining Lemma \ref{lem:gen} and Lemma \ref{lem:boundedsum}.
\end{proof}

\subsection{Bounds for Specific Algorithms}

Next, we plug in the regret guarantees we have established in Section \ref{sec:noregret} to yield the following accuracy and fairness generalization bounds:

\subsubsection{Exponential Weights}
Applying Theorem \ref{thm:avgloss} and Theorem \ref{thm:gen} to Theorem \ref{thm:exp-guarantee} (and set $q=\ln(\frac{|\cF|}{\delta})^{\frac{1}{4}} T^{\frac{7}{8}}$ for fairness), we get
\begin{corollary}
\label{cor:exp-weight}

With probability $1-\delta$, exponential weights with learning rate $\gamma = \sqrt{\frac{\ln(\vert\cF\vert)}{\totalRound}}$
achieves
\begin{enumerate}
    \item \textbf{Accuracy:}
    \[
    \E_{(x,y)\sim\cD}\left[\ell(\pi^{avg}(x),y)\right] \leq  \min\limits_{\pi \in \Pi_{\alpha}}\E_{(x,y)\sim\cD} \left[\ell(\pi(x),y)\right] + O\left(\frac{T^{-\frac{1}{4}}}{\alpha k}\sqrt{\ln\left(\frac{\vert \cF \vert}{\delta}\right)} \right).
    \]
    
    \item \textbf{Fairness:} $\pi^{avg}$ is $(\alpha + \lambda,\lambda)$-fair where
    \[
\lambda = O\left( \left(\frac{k}{\alpha}\right)^{\frac{1}{2}} \left(\ln\left(\frac{|\cF|}{\delta}\right)\right)^\frac{1}{4} T^{-\frac{1}{8}} \right).
\]

\end{enumerate}
\end{corollary}

\subsubsection{Follow-The-Perturbed-Leader}
Similarly for $\ftpl$, we can get the following generalization guarantee for $\ftpl$ by noting that, with probability $1-\delta$, \[\Regret_{\acc} = \frac{1}{\alpha} O \left(n^\frac{3}{4} \ln(|\cF|)^\frac{1}{2}T^\frac{5}{9} + \sqrt{T\ln\left(\frac{Tk}{\delta}\right)}  + k T^{\frac{7}{9}}\right)\] and \[\Regret_{\fair} = O\left(kn^\frac{3}{4} \ln\left(\frac{|\cF| Tk}{\delta}\right)^\frac{1}{2}  T^\frac{7}{9}\right).\] We set $q=k^\frac{1}{2} n^\frac{3}{8}  \ln\left(\frac{|\cF|Tk}{\delta}\right)^{\frac{1}{4}} T^{\frac{8}{9}}$ for fairness here, to yield to following guarantee:
\begin{corollary}
\label{cor:ftpl}
Using $\ftpl$ from \cite{SyrgkanisKS16} with a separator set of size $n$, with probability $1-\delta$, the average policy $\pi^{avg}$ has the following guarantee:
\begin{enumerate}
    \item \textbf{Accuracy:}\[
    \E_{(x,y)\sim\cD}\left[\ell(\pi^{avg}(x),y)\right] \leq  \min\limits_{\pi \in \Pi_{\alpha}}\E_{(x,y)\sim\cD} \left[\ell(\pi(x),y)\right] + O\left(\frac{n^\frac{3}{4} \ln(|\cF|)^\frac{1}{2}T^{\frac{-4}{9}} }{\alpha k} +  \sqrt{\frac{\ln\left(\frac{Tk}{\delta}\right)}{T}} +T^{\frac{-2}{9}} \right).
    \]
    \item \textbf{Fairness:} $\pi^{avg}$ is $(\alpha + \lambda,\lambda)$-fair where
    \[
\lambda = O\left( k^\frac{1}{2} n^\frac{3}{8}  \ln\left(\frac{|\cF|Tk}{\delta}\right)^{\frac{1}{4}} T^{-\frac{1}{9}}\right).
\]

\end{enumerate}
\end{corollary}

\begin{remark}
Corollaries \ref{cor:exp-weight} and \ref{cor:ftpl} imply that the average policy $\pi^{avg}$ has a non-trivial accuracy guarantee and a fairness generalization bound despite having much weaker feedback than considered in the setting of~\cite{YonaR18}, albeit worse rate than their uniform convergence bound. 
\end{remark}

\section{Conclusion and Future Directions} 
In this work, we have removed several binding restrictions in the context of learning with individual fairness. We hope that relieving the metric assumption as well as the assumption regarding full access to the similarity measure, and only requiring the auditor to detect a single violation for every time interval, will be helpful in making individual fairness more achievable and easier to implement in practice. As for future work - first of all, it would be interesting to explore the interaction with different models of feedback (one natural variant being one-sided feedback). Second, thinking about a model where the auditor only has access to binary decisions may be helpful in further closing the gap to practical use. 
Third, as most of the literature on individual fairness (including this work) is decoupling the similarity measure from the distribution over the target variable, it would be desirable to try to explore and quantify the compatibility of the two in specific instances.

\section{Acknowledgments}
We thank Sampath Kannan, Akshay Krishnamurthy, Katrina Ligett, and Aaron Roth for helpful conversations at an early stage of this work. Part of this work was done while YB, CJ, and ZSW were visiting the Simons Institute for the Theory of Computing.
YB is supported in part by Israel Science Foundation (ISF) grant \#1044/16, the United States Air Force and DARPA under contracts FA8750-16-C-0022 and FA8750-19-2-0222, the Federmann Cyber Security Center in conjunction with the Israel National Cyber Directorate, and the Apple Scholars in AI/ML PhD Fellowship. CJ is supported in part by NSF grant AF-1763307. ZSW is supported in part by the NSF FAI Award  \#1939606, an Amazon Research Award, a Google Faculty Research Award, a J.P. Morgan Faculty Award, a Facebook Research Award, and a Mozilla Research Grant. Any opinions, findings and conclusions or recommendations expressed in this material are those of the author(s) and do not necessarily reflect the views of the United States Air Force and DARPA. 

\bibliographystyle{plainnat}
\bibliography{refs.bib}

\begin{thebibliography}{33}
\providecommand{\natexlab}[1]{#1}
\providecommand{\url}[1]{\texttt{#1}}
\expandafter\ifx\csname urlstyle\endcsname\relax
  \providecommand{\doi}[1]{doi: #1}\else
  \providecommand{\doi}{doi: \begingroup \urlstyle{rm}\Url}\fi

\bibitem[Agarwal et~al.(2018)Agarwal, Beygelzimer, Dud{\'{\i}}k, Langford, and
  Wallach]{AgarwalBD0W18}
Alekh Agarwal, Alina Beygelzimer, Miroslav Dud{\'{\i}}k, John Langford, and
  Hanna~M. Wallach.
\newblock A reductions approach to fair classification.
\newblock In \emph{Proceedings of the 35th International Conference on Machine
  Learning, {ICML} 2018, Stockholmsm{\"{a}}ssan, Stockholm, Sweden, July 10-15,
  2018}, pages 60--69, 2018.
\newblock URL \url{http://proceedings.mlr.press/v80/agarwal18a.html}.

\bibitem[Agarwal et~al.(2019)Agarwal, Gonen, and Hazan]{AgarwalGH19}
Naman Agarwal, Alon Gonen, and Elad Hazan.
\newblock Learning in non-convex games with an optimization oracle.
\newblock In Alina Beygelzimer and Daniel Hsu, editors, \emph{Conference on
  Learning Theory, {COLT} 2019, 25-28 June 2019, Phoenix, AZ, {USA}}, volume~99
  of \emph{Proceedings of Machine Learning Research}, pages 18--29. {PMLR},
  2019.
\newblock URL \url{http://proceedings.mlr.press/v99/agarwal19a.html}.

\bibitem[Arora et~al.(2012)Arora, Hazan, and Kale]{arora2012multiplicative}
Sanjeev Arora, Elad Hazan, and Satyen Kale.
\newblock The multiplicative weights update method: a meta-algorithm and
  applications.
\newblock \emph{Theory of Computing}, 8\penalty0 (1):\penalty0 121--164, 2012.

\bibitem[Cao and Liu(2019)]{CaoL19}
Xuanyu Cao and K.~J.~Ray Liu.
\newblock Online convex optimization with time-varying constraints and bandit
  feedback.
\newblock \emph{{IEEE} Trans. Autom. Control.}, 64\penalty0 (7):\penalty0
  2665--2680, 2019.
\newblock \doi{10.1109/TAC.2018.2884653}.
\newblock URL \url{https://doi.org/10.1109/TAC.2018.2884653}.

\bibitem[Cesa-Bianchi et~al.(1997)Cesa-Bianchi, Freund, Haussler, Helmbold,
  Schapire, and Warmuth]{cesa}
Nicol\`{o} Cesa-Bianchi, Yoav Freund, David Haussler, David~P. Helmbold,
  Robert~E. Schapire, and Manfred~K. Warmuth.
\newblock How to use expert advice.
\newblock \emph{J. ACM}, 44\penalty0 (3):\penalty0 427–485, May 1997.
\newblock ISSN 0004-5411.
\newblock \doi{10.1145/258128.258179}.
\newblock URL \url{https://doi.org/10.1145/258128.258179}.

\bibitem[Cesa{-}Bianchi et~al.(2004)Cesa{-}Bianchi, Conconi, and
  Gentile]{CesaBianchi04}
Nicol{\`{o}} Cesa{-}Bianchi, Alex Conconi, and Claudio Gentile.
\newblock On the generalization ability of on-line learning algorithms.
\newblock \emph{{IEEE} Trans. Inf. Theory}, 50\penalty0 (9):\penalty0
  2050--2057, 2004.
\newblock \doi{10.1109/TIT.2004.833339}.
\newblock URL \url{https://doi.org/10.1109/TIT.2004.833339}.

\bibitem[Chouldechova(2017)]{chouldechova}
Alexandra Chouldechova.
\newblock Fair prediction with disparate impact: A study of bias in recidivism
  prediction instruments.
\newblock \emph{Big Data}, Special Issue on Social and Technical Trade-Offs,
  2017.

\bibitem[Chouldechova and Roth(2018)]{ChouR}
Alexandra Chouldechova and Aaron Roth.
\newblock The frontiers of fairness in machine learning.
\newblock \emph{CoRR}, abs/1810.08810, 2018.
\newblock URL \url{http://arxiv.org/abs/1810.08810}.

\bibitem[Corbett{-}Davies and Goel(2018)]{CorbettDaviesGo18}
Sam Corbett{-}Davies and Sharad Goel.
\newblock The measure and mismeasure of fairness: {A} critical review of fair
  machine learning.
\newblock arXiv, 2018.

\bibitem[Dwork et~al.(2012)Dwork, Hardt, Pitassi, Reingold, and
  Zemel]{individualfairness}
Cynthia Dwork, Moritz Hardt, Toniann Pitassi, Omer Reingold, and Richard~S.
  Zemel.
\newblock Fairness through awareness.
\newblock In \emph{Innovations in Theoretical Computer Science 2012, Cambridge,
  MA, USA, January 8-10, 2012}, pages 214--226, 2012.
\newblock \doi{10.1145/2090236.2090255}.
\newblock URL \url{https://doi.org/10.1145/2090236.2090255}.

\bibitem[Freund and Schapire(1997)]{FreundS97}
Yoav Freund and Robert~E. Schapire.
\newblock A decision-theoretic generalization of on-line learning and an
  application to boosting.
\newblock \emph{J. Comput. Syst. Sci.}, 55\penalty0 (1):\penalty0 119--139,
  1997.
\newblock \doi{10.1006/jcss.1997.1504}.
\newblock URL \url{https://doi.org/10.1006/jcss.1997.1504}.

\bibitem[Gillen et~al.(2018)Gillen, Jung, Kearns, and Roth]{GillenJKR18}
Stephen Gillen, Christopher Jung, Michael~J. Kearns, and Aaron Roth.
\newblock Online learning with an unknown fairness metric.
\newblock In \emph{Advances in Neural Information Processing Systems 31: Annual
  Conference on Neural Information Processing Systems 2018, NeurIPS 2018, 3-8
  December 2018, Montr{\'{e}}al, Canada.}, pages 2605--2614, 2018.
\newblock URL
  \url{http://papers.nips.cc/paper/7526-online-learning-with-an-unknown-fairness-metric}.

\bibitem[Gupta and Kamble(2019)]{GuptaK19}
Swati Gupta and Vijay Kamble.
\newblock Individual fairness in hindsight.
\newblock In \emph{Proceedings of the 2019 {ACM} Conference on Economics and
  Computation, {EC} 2019, Phoenix, AZ, USA, June 24-28, 2019}, pages 805--806,
  2019.
\newblock \doi{10.1145/3328526.3329605}.
\newblock URL \url{https://doi.org/10.1145/3328526.3329605}.

\bibitem[Hardt et~al.(2016)Hardt, Price, and Srebro]{hardt16}
Moritz Hardt, Eric Price, and Nathan Srebro.
\newblock Equality of opportunity in supervised learning.
\newblock In \emph{Neural Information Processing Systems (NIPS)}, 2016.

\bibitem[H{\'{e}}bert{-}Johnson et~al.(2018)H{\'{e}}bert{-}Johnson, Kim,
  Reingold, and Rothblum]{Hebert-JohnsonK18}
{\'{U}}rsula H{\'{e}}bert{-}Johnson, Michael~P. Kim, Omer Reingold, and Guy~N.
  Rothblum.
\newblock Multicalibration: Calibration for the (computationally-identifiable)
  masses.
\newblock In \emph{Proceedings of the 35th International Conference on Machine
  Learning, {ICML} 2018, Stockholmsm{\"{a}}ssan, Stockholm, Sweden, July 10-15,
  2018}, pages 1944--1953, 2018.
\newblock URL \url{http://proceedings.mlr.press/v80/hebert-johnson18a.html}.

\bibitem[Helmbold and Warmuth(1995)]{HelmboldWarmuth}
David~P. Helmbold and Manfred~K. Warmuth.
\newblock On weak learning.
\newblock \emph{J. Comput. Syst. Sci.}, 50\penalty0 (3):\penalty0 551--573,
  1995.
\newblock \doi{10.1006/jcss.1995.1044}.
\newblock URL \url{https://doi.org/10.1006/jcss.1995.1044}.

\bibitem[Ilvento(2019)]{Ilvento}
Christina Ilvento.
\newblock Metric learning for individual fairness.
\newblock \emph{CoRR}, abs/1906.00250, 2019.
\newblock URL \url{http://arxiv.org/abs/1906.00250}.

\bibitem[Jabbari et~al.(2017)Jabbari, Joseph, Kearns, Morgenstern, and
  Roth]{JabbariJKMR17}
Shahin Jabbari, Matthew Joseph, Michael~J. Kearns, Jamie Morgenstern, and Aaron
  Roth.
\newblock Fairness in reinforcement learning.
\newblock In \emph{Proceedings of the 34th International Conference on Machine
  Learning, {ICML} 2017, Sydney, NSW, Australia, 6-11 August 2017}, pages
  1617--1626, 2017.
\newblock URL \url{http://proceedings.mlr.press/v70/jabbari17a.html}.

\bibitem[Jenatton et~al.(2016)Jenatton, Huang, and
  Archambeau]{jenatton2016adaptive}
Rodolphe Jenatton, Jim Huang, and C{\'e}dric Archambeau.
\newblock Adaptive algorithms for online convex optimization with long-term
  constraints.
\newblock In \emph{International Conference on Machine Learning}, pages
  402--411. PMLR, 2016.

\bibitem[Joseph et~al.(2016)Joseph, Kearns, Morgenstern, and Roth]{JosephKMR16}
Matthew Joseph, Michael~J. Kearns, Jamie~H. Morgenstern, and Aaron Roth.
\newblock Fairness in learning: Classic and contextual bandits.
\newblock In \emph{Advances in Neural Information Processing Systems 29: Annual
  Conference on Neural Information Processing Systems 2016, December 5-10,
  2016, Barcelona, Spain}, pages 325--333, 2016.
\newblock URL
  \url{http://papers.nips.cc/paper/6355-fairness-in-learning-classic-and-contextual-bandits}.

\bibitem[Joseph et~al.(2018)Joseph, Kearns, Morgenstern, Neel, and
  Roth]{JosephKMNR18}
Matthew Joseph, Michael~J. Kearns, Jamie Morgenstern, Seth Neel, and Aaron
  Roth.
\newblock Meritocratic fairness for infinite and contextual bandits.
\newblock In \emph{Proceedings of the 2018 {AAAI/ACM} Conference on AI, Ethics,
  and Society, {AIES} 2018, New Orleans, LA, USA, February 02-03, 2018}, pages
  158--163, 2018.
\newblock \doi{10.1145/3278721.3278764}.
\newblock URL \url{https://doi.org/10.1145/3278721.3278764}.

\bibitem[Jung et~al.(2019)Jung, Kearns, Neel, Roth, Stapleton, and
  Wu]{subjectiveFairness}
Christopher Jung, Michael~J. Kearns, Seth Neel, Aaron Roth, Logan Stapleton,
  and Zhiwei~Steven Wu.
\newblock Eliciting and enforcing subjective individual fairness.
\newblock \emph{CoRR}, abs/1905.10660, 2019.
\newblock URL \url{http://arxiv.org/abs/1905.10660}.

\bibitem[Kannan et~al.(2017)Kannan, Kearns, Morgenstern, Pai, Roth, Vohra, and
  Wu]{KannanKMPRVW17}
Sampath Kannan, Michael~J. Kearns, Jamie Morgenstern, Mallesh~M. Pai, Aaron
  Roth, Rakesh~V. Vohra, and Zhiwei~Steven Wu.
\newblock Fairness incentives for myopic agents.
\newblock In \emph{Proceedings of the 2017 {ACM} Conference on Economics and
  Computation, {EC} '17, Cambridge, MA, USA, June 26-30, 2017}, pages 369--386,
  2017.
\newblock \doi{10.1145/3033274.3085154}.
\newblock URL \url{https://doi.org/10.1145/3033274.3085154}.

\bibitem[Kearns et~al.(2018)Kearns, Neel, Roth, and Wu]{KearnsNRW18}
Michael~J. Kearns, Seth Neel, Aaron Roth, and Zhiwei~Steven Wu.
\newblock Preventing fairness gerrymandering: Auditing and learning for
  subgroup fairness.
\newblock In \emph{Proceedings of the 35th International Conference on Machine
  Learning, {ICML} 2018, Stockholmsm{\"{a}}ssan, Stockholm, Sweden, July 10-15,
  2018}, pages 2569--2577, 2018.
\newblock URL \url{http://proceedings.mlr.press/v80/kearns18a.html}.

\bibitem[Kim et~al.(2018)Kim, Reingold, and Rothblum]{KimRR18}
Michael~P. Kim, Omer Reingold, and Guy~N. Rothblum.
\newblock Fairness through computationally-bounded awareness.
\newblock In \emph{Advances in Neural Information Processing Systems 31: Annual
  Conference on Neural Information Processing Systems 2018, NeurIPS 2018, 3-8
  December 2018, Montr{\'{e}}al, Canada}, pages 4847--4857, 2018.
\newblock URL
  \url{http://papers.nips.cc/paper/7733-fairness-through-computationally-bounded-awareness}.

\bibitem[Kleinberg et~al.(2017)Kleinberg, Mullainathan, and
  Raghavan]{kleinberg2017inherent}
Jon Kleinberg, Sendhil Mullainathan, and Manish Raghavan.
\newblock Inherent trade-offs in the fair determination of risk scores.
\newblock In \emph{Proceedings of the 8th Innovations in Theoretical Computer
  Science Conference}, 2017.

\bibitem[Mahdavi et~al.(2012)Mahdavi, Jin, and Yang]{MahdaviJY12}
Mehrdad Mahdavi, Rong Jin, and Tianbao Yang.
\newblock Trading regret for efficiency: online convex optimization with long
  term constraints.
\newblock \emph{J. Mach. Learn. Res.}, 13:\penalty0 2503--2528, 2012.
\newblock URL \url{http://dl.acm.org/citation.cfm?id=2503322}.

\bibitem[Podesta et~al.(2014)Podesta, Pritzker, Moniz, Holdren, and
  Zients]{podesta2014big}
John Podesta, Penny Pritzker, Ernest~J. Moniz, John Holdren, and Jefrey Zients.
\newblock Big data: Seizing opportunities and preserving values.
\newblock 2014.

\bibitem[Rothblum and Yona(2018)]{YonaR18}
Guy~N. Rothblum and Gal Yona.
\newblock Probably approximately metric-fair learning.
\newblock In \emph{Proceedings of the 35th International Conference on Machine
  Learning, {ICML} 2018, Stockholmsm{\"{a}}ssan, Stockholm, Sweden, July 10-15,
  2018}, pages 5666--5674, 2018.
\newblock URL \url{http://proceedings.mlr.press/v80/yona18a.html}.

\bibitem[Suggala and Netrapalli(2019)]{ftpl2}
Arun~Sai Suggala and Praneeth Netrapalli.
\newblock Online non-convex learning: Following the perturbed leader is
  optimal.
\newblock \emph{CoRR}, abs/1903.08110, 2019.
\newblock URL \url{http://arxiv.org/abs/1903.08110}.

\bibitem[Syrgkanis et~al.(2016)Syrgkanis, Krishnamurthy, and
  Schapire]{SyrgkanisKS16}
Vasilis Syrgkanis, Akshay Krishnamurthy, and Robert~E. Schapire.
\newblock Efficient algorithms for adversarial contextual learning.
\newblock In \emph{Proceedings of the 33nd International Conference on Machine
  Learning, {ICML} 2016, New York City, NY, USA, June 19-24, 2016}, pages
  2159--2168, 2016.
\newblock URL \url{http://proceedings.mlr.press/v48/syrgkanis16.html}.

\bibitem[Yu and Neely(2020)]{0002N20}
Hao Yu and Michael~J. Neely.
\newblock A low complexity algorithm with o({\(\surd\)}t) regret and {O(1)}
  constraint violations for online convex optimization with long term
  constraints.
\newblock \emph{J. Mach. Learn. Res.}, 21:\penalty0 1:1--1:24, 2020.
\newblock URL \url{http://jmlr.org/papers/v21/16-494.html}.

\bibitem[Yu et~al.(2017)Yu, Neely, and Wei]{YuNW17}
Hao Yu, Michael~J. Neely, and Xiaohan Wei.
\newblock Online convex optimization with stochastic constraints.
\newblock In Isabelle Guyon, Ulrike von Luxburg, Samy Bengio, Hanna~M. Wallach,
  Rob Fergus, S.~V.~N. Vishwanathan, and Roman Garnett, editors, \emph{Advances
  in Neural Information Processing Systems 30: Annual Conference on Neural
  Information Processing Systems 2017, December 4-9, 2017, Long Beach, CA,
  {USA}}, pages 1428--1438, 2017.
\newblock URL
  \url{https://proceedings.neurips.cc/paper/2017/hash/da0d1111d2dc5d489242e60ebcbaf988-Abstract.html}.

\end{thebibliography}

\newpage
\appendix
\section{Omitted Details from Section \ref{sec:gen}}
\label{app:gen}

\thmAvgLoss*
\begin{proof}
Fix $\pi^* \in \argmin\limits_{\pi\in \Pi_\alpha}\E\limits_{(x,y)\sim\cD}\left[\ell(\pi(x),y)\right]$. Note that $\pi^* \in Q_\alpha$ regardless of the realization of the sequence $((x^t, y^t))_{t=1}^T$. 
\begin{lemma}
For any fixed randomness of the algorithm, we must have
  \[
 \Pr_{\substack{(\bx^\round, \by^\round)_{\round=1}^{\totalRound}\\ \pi^\round = \Algorithm(\cdots)}}\left[\left\vert \sum_{\round=1}^\totalRound 
 \sum_{\subRound=1}^{\totalSubRound} \ell\left(\pi^\round\left(x^\round_{\subRound}\right), y^\round_{\subRound}\right) - \E_{({{\bx}}^{\prime\round}, {{\by}}^{\prime\round})_{\round=1}^{\totalRound}}\left[\sum_{\round=1}^\totalRound 
 \sum_{\subRound=1}^{\totalSubRound} \ell\left(\pi^\round\left({x'}^\round_{\subRound}\right), {y}^{\prime\round}_{\subRound}\right)
        \right]
 \right\vert \ge \gamma \right] \le 2\exp\left(\frac{-\gamma^2}{2\totalSubRound^2\totalRound}\right)
 \]
\end{lemma}
\begin{proof}

Define $\left(A^\round \right)_{\round=1}^\totalRound$ as
\[
A^\round = \sum_{j=1}^\round
 \sum_{\subRound=1}^{\totalSubRound} \ell\left(\pi^j\left(x^j_{\subRound}\right), y^j_{\subRound}\right) - \E_{({\bx}^{\prime j}, {\by}^{\prime j})_{j=1}^{\round}}\left[\sum_{j=1}^\round 
 \sum_{\subRound=1}^{\totalSubRound} \ell\left(\pi^j\left({x}^{\prime j}_{\subRound}\right), {y}^{\prime j}_{\subRound}\right)
        \right].
\]
Note that $\left(A^\round \right)_{\round=1}^\totalRound$ is a martingale: $\E[A^\round | A^1, \dots, A^{\round-1}] = A^{\round-1}$ because $\pi^t$ is deterministically determined in each round as a result of the fixed randomness. Also, we have $\vert A^\round - A^{\round-1} \vert \le \totalSubRound$. Therefore, applying Azuma's inequality yields 
\[
\Pr\left[ \vert A^\totalRound - A^1 \vert \ge \gamma\right] \le 2\exp\left(\frac{-\gamma^2}{2\totalSubRound^2\totalRound}\right).
\]
\end{proof}

Applying the Chernoff bound, we get a similar concentration bound for $\pi^*$:
\[
 \Pr_{(\bx^\round, \by^\round)_{\round=1}^{\totalRound} }\left[\left\vert \sum_{\round=1}^\totalRound
 \sum_{\subRound=1}^{\totalSubRound} \ell\left(\pi^*\left(x^\round_{\subRound}\right), y^\round_{\subRound}\right) - \E_{({\bx}^{\prime\round}, {\by}^{\prime\round})_{\round=1}^{\totalRound}}\left[\sum_{\round=1}^\totalRound 
 \sum_{\subRound=1}^{\totalSubRound} \ell\left(\pi^*\left({x}^{\prime\round}_{\subRound}\right), {y}^{\prime\round}_{\subRound}\right)
        \right]
 \right\vert \ge \gamma \right] \le 2\exp\left(\frac{-\gamma^2}{2\totalSubRound^2\totalRound}\right).
\]

Next, using triangle inequality, with probability $1-\delta$ only over the randomness of $(x^t,y^t)_{t=1}^T$, it holds that for any fixed randomness of the algorithm
\begin{align}
&\E_{(\bx^\round, \by^\round)_{\round=1}^\totalRound}\left[\sum_{\round=1}^\totalRound \sum_{\subRound=1}^{\totalSubRound} \ell\left(\pi^\round\left(x^\round_{\subRound}\right), y^\round_{\subRound}\right)\right] - \E_{(\bx^\round, \by^\round)_{\round=1}^\totalRound}\left[\sum_{\round=1}^\totalRound \sum_{\subRound=1}^{\totalSubRound} \ell(\pi^*(x^\round_{\subRound}), y^\round_{\subRound})\right]\nonumber\\ 
&\le \sum_{\round=1}^\totalRound \sum_{\subRound=1}^{\totalSubRound} \ell\left(\pi^\round\left(x^\round_{\subRound}\right), y^\round_{\subRound}\right) - \sum_{\round=1}^\totalRound \sum_{\subRound=1}^{\totalSubRound} \ell(\pi^*(x^\round_{\subRound}), y^\round_{\subRound}) + 2\sqrt{\ln\left(\frac{4}{\delta}\right){2\totalSubRound^2\totalRound}}\label{eqn:azuma-acc}
\end{align}

Suppose we write $\seed$ to denote the random seed used for the algorithm. For simplicity, write $E_1\left[\left(x^t, y^t)\right)_{t=1}^T, \seed\right]$ to denote the event that inequality \eqref{eqn:azuma-acc} doesn't hold true.

And we write $E_2\left[\left((x^t, y^t)\right)_{t=1}^T, \seed\right]$ to denote the event that \[
    \sum_{t=1}^T \misclassLoss(\pi^t, (x^t, y^t)) - \min_{\pi^* \in Q_{\alpha}} \sum_{t=1}^T \misclassLoss(\pi^*, x^t, y^t) > \Regret_{\acc}
\]

We have that 
\begin{align*}
    &\forall \seed: \Pr_{(x^t, y^t)}\left[E_1\left[\left((x^t, y^t)\right)_{t=1}^T, \seed\right] \middle\vert \seed\right] \le \delta\\
    &\forall ((x^t, y^t))_{t=1}^T: \Pr_{\seed}\left[E_2\left[\left(((x^t, y^t)\right)_{t=1}^T, \seed\right] \middle\vert ((x^t, y^t))_{t=1}^T\right] \le \delta
\end{align*}

Therefore, we can use union bound to show
\begin{align*}
    &\Pr_{\seed, ((x^t, y^t))_{t=1}^T}\left[E_1\left[\left((x^t, y^t)\right)_{t=1}^T, \seed\right]\quad\text{or}\quad E_2\left[\left((x^t, y^t)\right)_{t=1}^T, \seed\right]\right]\\
    &\le \Pr_{\seed, ((x^t, y^t))_{t=1}^T}\left[E_1\left[\left((x^t, y^t)\right)_{t=1}^T, \seed\right]\right] + \Pr_{\seed, ((x^t, y^t))_{t=1}^T}\left[E_2\left[\left((x^t, y^t)\right)_{t=1}^T, \seed\right]\right]\\
    &\le \int_{((x^t, y^t))_{t=1}^T} \Pr_{\seed}\left[E_1\left(\left(x^t, y^t)\right)_{t=1}^T, \seed\right] \middle\vert ((x^t, y^t))_{t=1}^T\right]d\cD^{kT}(((x^t, y^t))_{t=1}^T) \\
    &+ \int_{\seed} \Pr_{((x^t, y^t))_{t=1}^T}\left[E_2\left[\left((x^t, y^t)\right)_{t=1}^T, \seed\right] \middle\vert \seed\right]d\cP(\seed)\\
    &\le 2\delta.
\end{align*}
In other words, we have that with probability $1-2\delta$,
\begin{align*}
&\frac{1}{kT}\E_{(\bx^\round, \by^\round)_{\round=1}^\totalRound}\left[\sum_{\round=1}^\totalRound \sum_{\subRound=1}^{\totalSubRound} \ell\left(\pi^\round\left(x^\round_{\subRound}\right), y^\round_{\subRound}\right)\right] - \frac{1}{kT}\E_{(\bx^\round, \by^\round)_{\round=1}^\totalRound}\left[\sum_{\round=1}^\totalRound \sum_{\subRound=1}^{\totalSubRound} \ell(\pi^*(x^\round_{\subRound}), y^\round_{\subRound})\right]\\ 
&\le \sum_{\round=1}^\totalRound \sum_{\subRound=1}^{\totalSubRound} \ell\left(\pi^\round\left(x^\round_{\subRound}\right), y^\round_{\subRound}\right) - \sum_{\round=1}^\totalRound \sum_{\subRound=1}^{\totalSubRound} \ell(\pi^*(x^\round_{\subRound}), y^\round_{\subRound}) + 2\sqrt{\ln\left(\frac{4}{\delta}\right){2\totalSubRound^2\totalRound}}\\
&\le \frac{1}{kT}\Regret_{\acc} + \frac{1}{kT} 2\sqrt{\ln\left(\frac{4}{\delta}\right){2\totalSubRound^2\totalRound}}.
\end{align*}

For the left hand side of the inequality, observe that the following holds:
\begin{align}
\E_{(x,y)\sim\cD} \left[\ell(\pi^{avg}(x),y)\right] &= \frac{1}{\totalSubRound\totalRound}\sum_{\subRound=1}^\totalSubRound \sum_{i=1}^\totalRound\E_{(x,y)\sim\cD}\left[\E_{\pi\sim\mathbb{U}\{\pi^1,\dots,\pi^T\}}\left[\ell(\pi(x),y)\right]\right]\label{tra:linear}\\
&=
\frac{1}{\totalSubRound\totalRound}\sum_{\subRound=1}^\totalSubRound \sum_{i=1}^\totalRound\E_{(x,y)\sim\cD}\left[\frac{1}{\totalRound}\sum_{\round=1}^\totalRound\ell(\pi^\round(x),y)\right]\notag\\
&=
\frac{1}{\totalSubRound\totalRound}\sum_{\subRound=1}^\totalSubRound \sum_{\round=1}^\totalRound\E_{(x,y)\sim\cD}\left[\frac{1}{\totalRound}\sum_{i=1}^\totalRound\ell(\pi^\round(x),y)\right]\notag\\
&=
\frac{1}{\totalSubRound\totalRound}\sum_{\subRound=1}^\totalSubRound \sum_{\round=1}^\totalRound\E_{(x,y)\sim\cD}\left[\ell(\pi^\round(x),y)\right]\notag\\
&=
\frac{1}{\totalSubRound\totalRound}\E_{(\bx^\round, \by^\round)_{\round=1}^\totalRound\sim_{i.i.d.}\cD^{\totalSubRound\totalRound}}\left[\sum_{\round=1}^\totalRound \sum_{\subRound=1}^\totalSubRound\ell(\pi^\round(x^\round_\subRound),y^\round_\subRound)\right]\notag
\end{align}

Transition \ref{tra:linear} stems from the fact that our misclassification loss $\ell$ is linear with respect to the base classifiers in $\mathcal{F}$. Hence, taking the uniform distribution over $\pi^1,\dots,\pi^\totalRound$ gives
\[
\ell(\pi^{avg}(x),y) = \E\limits_{\pi\sim\mathbb{U}\{\pi^1,\dots,\pi^T\}}\left[\ell(\pi(x),y)\right].
\]
\end{proof}

\textbf{Tightness of Bound in Observation \ref{obs:basic}}\\
Consider the following example:
$\cX = \{x_1,x_2,x_2\}$, $\cD\vert_{\cX} = \mathbb{U}\{\cX\}$ (i.e. uniform distribution over $\cX$), $\alpha = 0.1$. $\mathcal{F} = \{h_1,h_2\}$ given by:
\begin{align*}
h_1(x_1) = 1,~h_1(x_2) = 0,~h_1(x_3) = 0\\ 
h_2(x_1) = 1,~h_2(x_2) = 1,~h_2(x_3) = 0   
\end{align*}

Also, assume the dissimilarity measure by the auditor is: 
\[
\metric(x_1,x_2) = 0,~\metric(x_2,x_3) = 0,~\metric(x_3,x_1) = 1.
\]

Assume the algorithm deploys policies in the following manner:
$\pi^\round = \begin{cases}h_1 &\text{$\round$~is~odd}\\h_2&\text{$\round$~is~even}
\end{cases}$

Both $h_1,h_2$ are exactly $(\alpha,\frac{1}{8})-fair$. If $\totalRound$ is even, $\pi^{avg}$ is exactly $(\alpha,\frac{1}{4})$-fair.

\end{document}